\documentclass{article}
\pdfoutput=1

     \usepackage[final]{neurips_2020}

\usepackage[utf8]{inputenc} 
\usepackage[T1]{fontenc}    
\usepackage{hyperref}       
\usepackage{url}            
\usepackage{booktabs}       
\usepackage{amsfonts}       
\usepackage{nicefrac}       
\usepackage{microtype}      
\usepackage{color}
\usepackage{makecell}
\usepackage{caption}

\usepackage[ruled,vlined]{algorithm2e}
\usepackage{amsmath}
\usepackage{amsthm}
\usepackage{amssymb}

\usepackage{hyperref}
\hypersetup{colorlinks,linkcolor={blue},citecolor={blue},urlcolor={red}} 

\newtheorem{theorem}{Theorem}

\newtheorem{lemma}{Lemma}
\newtheorem{proposition}{Proposition}

\theoremstyle{definition}
\newtheorem{definition}{Definition}

\theoremstyle{assumption}
\newtheorem{assumption}{Assumption}

\newtheorem{remark}{Remark}

\title{Is Plug-in Solver Sample-Efficient for Feature-based Reinforcement Learning?}

\author{
	Qiwen Cui\\
	School of Mathematical Science,\\
	Peking University\\
	\texttt{cuiqiwen@pku.edu.cn}\\
	\And
	Lin F. Yang\\
	Electrical and Computer Engineering Department,\\
	University of California, Los Angles\\
	\texttt{linyang@ee.ucla.edu}
}

\begin{document}
	
	\maketitle
	
	\begin{abstract}
		It is believed that a model-based approach for reinforcement learning (RL) is the key to reduce sample complexity. However, the understanding of the sample optimality of model-based RL is still largely missing, even for the linear case. 
		This work considers sample complexity of finding an $\epsilon$-optimal policy in a Markov decision process (MDP) that admits a linear additive feature representation, given only access to a generative model. 
		We solve this problem via a plug-in solver approach, which builds an empirical model and plans in this empirical model via an arbitrary plug-in solver. 
		We prove that under the anchor-state assumption, which implies implicit non-negativity in the feature space, 
		the minimax sample complexity of finding an $\epsilon$-optimal policy in a $\gamma$-discounted MDP is $O(K/(1-\gamma)^3\epsilon^2)$, which only depends on the dimensionality $K$ of the feature space and has no dependence on the state or action space.
		We further extend our results to a relaxed setting where anchor-states may not exist and show that a plug-in approach can be sample efficient as well, providing a flexible approach to design model-based algorithms for RL. 
	\end{abstract}

	
	\section{Introduction}

	Reinforcement learning (RL) \citep{sutton2018reinforcement} is about learning to make optimal decisions in an unknown environment. It has been believed to be one of the key approaches to reach  artificial general intelligence.
	In recent years, RL achieves phenomenal empirical successes in many real-world applications, e.g., game-AI \citep{vinyals2017starcraft}, robot control \citep{duan2016benchmarking}, health-care \citep{li2018hybrid}.
	Most of these successful applications are based on a model-free approach, where the agent directly learns the value function of the environment.
	Despite superior performance, these algorithms usually take tremendous amount of samples. E.g. a typical model-free Atari-game agent takes about several hours of training data in order to perform well \citep{mnih2013playing}.
	Reducing sample complexity becomes a critical research topic of in RL. 
	
	It is well believed that model-based RL, where the agent learns the model of the environment and then performs planning in the model, is significantly more sample efficient than model-free RL. 
	Recent empirical advances also justify such a belief (e.g. \citep{kaiser2019model,wang2019benchmarking}).
	However, the understanding of model-based RL is still far from complete. E.g., how to deal with issues like model-bias and/or error compounding due to long horizon, in model-based RL is still an open question \citep{jiang2015dependence}, 
	especially with the presence of a function approximator (e.g. a neural network) on the model.
	In order to get a better understanding on these issues, we target on the sample complexity question of model-based RL from a very basic setting: feature-based RL.
	In feature-based RL, we are given a hand-crafted or learned low-dimensional feature-vector for each state-action pair and the transition model can be represented by a linear combination of the feature vectors. 
	Such a model recently has attracted much interest due to its provable guarantee with model-free algorithms (e.g. \citep{Yang2019,yang2019reinforcement,jin2019provably}). 
	In particular, we aim on answering the following fundamental question.
	
	\begin{center}
		\emph{Does a model-based approach on feature-based RL achieve near-optimal sample complexity?}
	\end{center}
	
	In particular, we focus on the generative model setting, where the agent is able to query samples freely from any chosen state-action pairs. 
	Such a model is proposed by \citep{kearns1999finite,kakade2003sample} 
	and gains a great deal of interests recently \citep{azar2013minimax,sidford2018near,Yang2019,zanette2019limiting}.
	Moreover, we focus on the \emph{plug-in solver approach}, which is probably the most intuitive and simplest approach for model-based RL: we first build an empirical model with an estimate of the transition probability matrix and then find a near optimal policy by planning in this empirical model via arbitrary plug-in solver. 
	In the tabular setting, where the state and action spaces, $\mathcal{S}$ and $\mathcal{A}$, are finite,
	\citep{azar2013minimax}  shows that the \emph{value estimation} of a plug-in approach is minimax optimal in samples. 
	In particular, they show that to obtain an $\epsilon$-optimal value, the number of samples required to estimate the model is $\widetilde{O}(|\mathcal{S}||\mathcal{A}|/\epsilon^2(1-\gamma)^3)$.\footnote{In $\widetilde{O}(f)$, $\log f$ factors are ignored.} 
	Very recently, \citep{agarwal2019optimality} proves that the \emph{policy estimation} is also minimax optimal and with the same sample complexity.
	Unfortunately, these results cannot be applied to the function approximation setting, especially when the number of states becomes infinity.

	In this paper, 
	we show that the plug-in solver approach do achieve \emph{near-optimal sample complexity} even in the feature-based setting, provided that the features are well conditioned.
	In particular, we show that under an anchor-state condition, where all features can be represented by the convex combination of some anchor-state features, 
	an $\epsilon$-optimal policy can be obtained from an approximate model with only   $\widetilde{O}(K/\epsilon^{2}(1-\gamma)^3)$ samples from the generative model, where $K$ is the feature dimension, independent of the size of state and action spaces. 
	Under a more relaxed setting on the features, we also prove that finding an $\epsilon$-optimal policy only needs $\widetilde{O}(K\cdot poly(1/(1-\gamma))/\epsilon^{2})$ samples. 
	To achieve our results, we observe that the value function actually lies in an one-dimensional manifold and thus we can construct a series of auxiliary MDPs to approximate the value function. This auxiliary MDP technique breaks the statistical dependence that impedes the analysis.
	We have also extended our techniques to other settings e.g. finite horizon MDP (FHMDP) and two-players turn-based stochastic games (2-TBSG).
	To our best knowledge, this work first proves that plug-in approach is sample-optimal for feature-based RL and we hope our  technique can boost analysis in broader settings.
	
	\begin{table}
		\caption{Sample complexity to compute $\epsilon$-optimal policy with generative model}
		\label{sample-table}
		\centering
		\begin{tabular}{llll}
			\toprule
			Algorithm     & Sample Complexity     & $\epsilon$-range & Problem type \\
			\midrule
			\makecell[l]{Empirical QVI\\ \citep{azar2013minimax}} & $\frac{|\mathcal{S}||\mathcal{A}|}{(1-\gamma)^3\epsilon^2}$ & $(0,\frac{1}{\sqrt{(1-\gamma)|\mathcal{S}|}}]$&Tabular MDP    \\
			\makecell[l]{Variance-reduced QVI\\ \citep{sidford2018variance}}     & $\frac{|\mathcal{S}||\mathcal{A}|}{(1-\gamma)^3\epsilon^2}$ &$(0,1]$   & Tabular MDP \\
			\makecell[l]{Empirical MDP\\ \citep{agarwal2019optimality}} & $\frac{|\mathcal{S}||\mathcal{A}|}{(1-\gamma)^3\epsilon^2}$ &$(0,\frac{1}{\sqrt{1-\gamma}}]$   &  Tabular MDP \\
			\makecell[l]{Perturbed empirical MDP\\ \citep{li2020breaking}} &  $\frac{|\mathcal{S}||\mathcal{A}|}{(1-\gamma)^3\epsilon^2}$ &$(0,\frac{1}{1-\gamma}]$   &  Tabular MDP \\
			\makecell[l]{QVI-MDVSS\\ \citep{Sidford2019}} & $\frac{|\mathcal{S}||\mathcal{A}|}{(1-\gamma)^3\epsilon^2}$ &$(0,1]$   &  Tabular TBSG \\
			\makecell[l]{OPPQ-Learning\\ \citep{Yang2019}} & $\frac{K}{(1-\gamma)^3\epsilon^2}$ & $(0,1]$ &  Linear MDP\\
			\makecell[l]{Two side PQ-Learning\\ \citep{jia2019feature}}&  $\frac{KL^2}{(1-\gamma)^4\epsilon^2}$ &$(0,1]$   &  Linear TBSG \\
			\makecell[l]{Empirical MDP\\ (This work)} & $\frac{K}{(1-\gamma)^3\epsilon^2}$ &$(0,\frac{1}{\sqrt{1-\gamma}}]$ & Linear MDP/TBSG \\
			\bottomrule
		\end{tabular}
		\caption*{Here $|\mathcal{S}|$ is the number of states, $|\mathcal{A}|$ is the number of actions, $\gamma$ is the discount factor, $K$ is the number of representative features, $\epsilon$ is the policy accuracy and $L$ is a coefficient measuring well-conditioned features.}
	\end{table}

	\section{Related Work}
	
	
	\paragraph{Generative Model} There is a line of research focusing  on improving the sample complexity with a generative model, e.g. \citep{kearns1999finite,kakade2003sample,azar2012sample,azar2013minimax,sidford2018near,sidford2018variance,Yang2019,Sidford2019,zanette2019limiting,li2020breaking}. 
	A classic algorithm under generative model setting is phased Q-learning \citep{kearns1999finite}. It uses $\widetilde{O}(|\mathcal{S}||\mathcal{A}|/\epsilon^2/\mathrm{poly}(1-\gamma))$ samples to find an $\epsilon$-optimal policy, which is sublinear to the model size $|\mathcal{S}|^2|\mathcal{A}|$. Sample complexity lower bound for generative model has been established in \citep{azar2013minimax,Yang2019,Sidford2019}. In particular, \citep{azar2013minimax} gives the first tight lower bound for unstructured discounted MDP. Later, this lower bound is generalized to feature-based MDP and two-players turn-based stochastic game in \citep{Yang2019,Sidford2019}. \citep{azar2013minimax} also proves a minimax sample complexity $\widetilde{O}(|\mathcal{S}||\mathcal{A}|/\epsilon^2(1-\gamma)^3)$ of model-based algorithm for value estimation via the total variance technique. However, the sample complexity of value estimation and policy estimation differs in a factor of $\widetilde{O}(1/(1-\gamma)^2)$ \citep{singh1994upper}. The first minimax policy estimation result is given in \citep{sidford2018near}, which proposes an model-free algorithm, known as Variance Reduced Q-value Iteration. This work has been extended to two-players turn-based stochastic game in \citep{Sidford2019,jia2019feature}. Recently, \citep{Yang2019} develops an sample-optimal algorithm called Optimal Phased Parametric Q-Learning for feature-based RL. Their result requires $\epsilon\in(0,1)$, while our result holds for $\epsilon\in(0,1/\sqrt{1-\gamma})$.
	Plug-in solver approach is proved to be sample-optimal for tabular case in \citep{agarwal2019optimality}, which develops the absorbing MDP technique. However, their approach can not be generalized to linear transition model. A very recently paper \citep{li2020breaking} develops a novel reward perturbation technique to remove the constraint on $\epsilon$ in tabular case.
	
	
	\paragraph{Function approximation RL} Linear function approximation and linear transition model has been long studied \citep{bradtke1996linear,melo2007q,munos2008finite,chu2011contextual,abbasi2011improved,jiang2015dependence,lever2016compressed,pires2016policy,azar2017minimax,jin2018q,lattimore2019learning,jin2019provably,Yang2019,zanette2019limiting,du2019provably}. Linear function approximation algorithms have been developed in \citep{bradtke1996linear,melo2007q,du2019provably}, which are also known as learning in the linear Q model. The concept of Bellman error in linear Q model is proposed in \citep{munos2008finite}, which obtains further analysis in \citep{jiang2017contextual,jiang2015dependence}.  \citep{Yang2019} proves that linear transition model is equivalent to linear Q model with zero Bellman error. A lot of work focuses on online learning setting, where the samples are collected by a learned policy. For tabular case, the tight regret lower bound is given in \citep{jaksch2010near} and the tight regret upper bound is given in \citep{azar2017minimax}. For linear feature RL, the best regret algorithms are given in \citep{jin2019provably} and \citep{yang2019reinforcement}, which study model-free algorithm and model-based algorithm respectively. However, the optimal regret for online linear RL is still unclear. For generative model setting, people care about sample complexity. Anchor-state assumption is the key to achieve minimax sample complexity, which is used in analyzing both the linear transition model and linear Q model \citep{Yang2019,zanette2019limiting}. Similar concept known as soft state aggregation is developed in \citep{singh1995reinforcement,duan2019state}. General function approximation has gathered increasing attention due to the impressive success of deep reinforcement learning. \citep{osband2014model} gives the regret bound of a model-based algorithm with general function approximation and \citep{jiang2017contextual,sun2018model} provides model-free algorithms and corresponding regret bounds. \citep{zanette2019limiting} makes linear $Q^*$ assumption and their result holds only for $\|\lambda\|_1\leq1+\frac{1}{H}$ so that the error will not amplify exponentially. Recently, \citep{wang2020provably} gives a efficient model-free algorithm without any structural assumption of the environment.
	
	\section{Preliminaries}
	
	In this section, we briefly introduce models that we will analyze in the following sections.
	
	\paragraph{Discounted Markov Decision Process}
	A discounted Markov Decision Process (DMDP or MDP) is described by the tuple $M=(\mathcal{S},\mathcal{A},P,r,\gamma)$, where $\mathcal{S}$ and $\mathcal{A}$ are the state and action spaces, $P$ is the probability transition matrix which specifies the dynamics of the system, $r$ is the reward function of each state-action pair and $\gamma\in(0,1)$ is the discount factor. Without loss of generality, we assume $r(s,a)\in[0,1],\forall (s,a)\in(\mathcal{S},\mathcal{A})$. The target of the agent is to find a stationary policy $\pi: \mathcal{S}\rightarrow\mathcal{A}$ that maximizes the discounted total reward from any initial state $s$:
	$$V^\pi(s):=\mathbb{E}\left[\sum_{t=0}^{\infty}\gamma^tr(s^t,\pi(s^t))~\bigg|~ s^0=s\right].$$
	We call $V^\pi:\mathcal{S}\rightarrow\mathbb{R}$ the value function and for finite state space, it can be regarded as a $|\mathcal{S}|$-dimensional vector as well. It is well known that an optimal stationary policy $\pi^*$ exists and that it maximizes the value function for all states:
	$V^*(s):=V^{\pi^*}(s)=\max_\pi V^\pi(s), \forall s\in\mathcal{S}.$
	The action-value or Q-function of policy $\pi$ is defined as
	$$Q^\pi(s,a):=\mathbb{E}\left[r(s^0,a^0)+\sum_{t=1}^{\infty}\gamma^tr(s^t,\pi(s^t))~\bigg|~s^0=s,a^0=a\right]=r(s,a)+\gamma P(s,a)V^\pi,$$
	where $P(s,a)$ is the $(s,a)$-th row of the transition matrix $P$. The optimal Q-function is denoted as $Q^*=Q^{\pi^*}$. Similarly, we have
	$Q^*(s,a)=\max_\pi Q^\pi(s,a),\forall (s,a)\in (\mathcal{S},\mathcal{A}).$
	It is straightforward to show that $V^\pi(s)\in[0,\frac{1}{1-\gamma}]$ and $Q^\pi(s,a)\in[0,\frac{1}{1-\gamma}]$. Our target is to find an $\epsilon$-optimal policy $\pi$ such that $V^\pi(s)\leq V^*(s)\leq V^\pi(s)+\epsilon,\forall s\in\mathcal{S}$ for some $\epsilon>0$.
	
	
	
	\paragraph{Feature-based Linear Transition Model}
	We consider the case that the transition matrix $P$ has a linear structure. Suppose the learning agent is given a feature function $\phi:\mathcal{S}\times\mathcal{A}\rightarrow\mathbb{R}^K$:
	$$\phi(s,a)=[\phi_1(s,a),\cdots,\phi_K(s,a)]$$
	The feature function provides information about the transition matrix via a linear additive model.
	
	\begin{definition}
		(Feature-based Linear Transition Model) For a transition probability matrix $P$, we say that $P$ admits a linear feature representation $\phi$ if for every $s,a,s'$,
		$$P(s'|s,a)=\sum_{k\in[K]}\phi_k(s,a)\psi_k(s'),$$
		for some unknown functions $\psi_1,\cdots,\psi_K:\mathcal{S}\rightarrow \mathbb{R}$.
	\end{definition}
	
	The linear transition model implies a low-rank factorization of the transition matrix $P=\Phi\Psi$ and one composite $\Phi$ is given as the features. In such a model, the number of unknown parameters is $K|\mathcal{S}|$, rather than $|\mathcal{S}|^2|\mathcal{A}|$ for unstructured MDP. The linear transition model is closely related to another widely studied feature-based MDP, i.e. linear Q model \citep{bradtke1996linear}. 
	
	
	
	\paragraph{Generative Model Oracle}
	Suppose we have access to a generative model which allows sampling from arbitrary state-action pair: $s'\sim P(\cdot |s,a)$. It is different from the online sampling oracle where a policy is used to collect data. To estimate the transition kernel $P$, we call the generative model $N$ times on each state-action pair in $\mathcal{K}$, where $|\mathcal{K}|=K$.\footnote{Our result holds for $|\mathcal{K}|=\widetilde{O}(K)$.} Then an estimate of the partial transition kernel $P_\mathcal{K}$ is:
	\begin{equation}
	\widehat{P}_\mathcal{K}(s'|s,a)=\frac{\mathrm{count}(s,a,s')}{N},
	\end{equation}
	where $\mathrm{count}(s,a,s')$ is the number of times the state $s'$ is sampled from $P(\cdot |s,a)$. The total sample size is $KN$. For tabular case, $\mathcal{K}$ is set to be $\mathcal{S}\times\mathcal{A}$ and $\widehat{P}=\
	\widehat{P}_\mathcal{K}$ is the estimate of full transition kernel. In the linear transition model, we have $K\ll |\mathcal{S}||\mathcal{A}|$ so that the sample complexity is greatly reduced. The selection of $\mathcal{K}$ and the estimation of full transition kernel $P$ will be discussed in the next section.
	
	\paragraph{Plug-in Solver Approach}
	
	In the empirical MDP $\widehat{M}$, we make use of a plug-in solver to get an approximately optimal policy. The plug-in solver receives $\widehat{M}$ with known transition distributions and outputs an $\epsilon_{\mathrm{PS}}$-optimal policy $\widehat{\pi}$. Our goal is to prove that $\widehat{\pi}$ is also an approximately optimal policy in the true MDP $M$.
	In fact, we can assume $\epsilon_{\mathrm{PS}}=0$ as an MDP can be exactly solved in polynomial time. Even so, we consider the general case as reaching an approximately optimal policy is much less time-consuming than finding the exact optimal policy. We regard this as a tradeoff between time complexity and policy optimality. Approximate dynamic programming methods like LSVI/FQI \citep{kearns1999finite} can utilize the features to achieve $\widetilde{O}(\mathrm{poly}(K(1-\gamma)^{-1}\epsilon^{-1}))$ computational complexity
	. In addition, learning algorithm `Optimal Phased Parametric Q-Learning' in \citep{Yang2019} can be used to do planning, which has computational complexity of $\widetilde{O}(K(1-\gamma)^{-3}\epsilon^{-2})$.
	
	\section{Empirical Model Construction}
	
	To construct an empirical MDP for linear transition model, the estimated transition kernel needs to be non-negative and sum to one. In our work, we propose a simple but effective method (Algorithm \ref{alg1}) to estimate the transition matrix.
	\begin{proposition}
		Assume we have a linear transition model with feature function $\phi:\mathcal{S}\times\mathcal{A}\rightarrow\mathbb{R}^K$. For a row basis index set $\mathcal{K}$ of $\phi$, there exists $\{\lambda_k^{s,a}\}$ such that 
		$\phi(s,a)=\sum_{k\in\mathcal{K}}\lambda_k^{s,a}\phi(s_k,a_k),\forall k\in\mathcal{K},(s,a)\in(\mathcal{S},\mathcal{A}),$
		Then $\widehat{P}(s'|s,a)=\sum_{k\in\mathcal{K}}\lambda_k^{s,a}\widehat{P}_\mathcal{K}(s'|s_k,a_k)$ satisfies
		\begin{enumerate}
			\item $\widehat{P}(s'|s,a)$ is an unbiased estimate of $P(s'|s,a)$,
			\item $\sum_{k\in\mathcal{K}}\lambda_k^{s,a}=1$ and $\sum_{s'}\widehat{P}(s'|s,a)=1$,
			\item if $\lambda_k^{s,a}\geq0,\forall k\in\mathcal{K},(s,a)\in(\mathcal{S},\mathcal{A})$, then $\widehat{P}(s'|s,a)\geq0,\forall (s,a,s')\in(\mathcal{S},\mathcal{A},\mathcal{S})$,
		\end{enumerate}
		where $\widehat{P}_\mathcal{K}$ is given in (1).
		\label{prop1}
	\end{proposition} 
	
	
	This proposition shows that the output of Algorithm \ref{alg1} is an eligible estimate of the transition kernel $P$, when a proper state-action set $\mathcal{K}$ is chosen. Different choices of $\mathcal{K}$ can lead to different $\widehat{P}$. To ensure the non-negativity of $\widehat{P}$, we use a special class of state-action set $\mathcal{K}$.
	
	\begin{assumption}
		(Anchor-state assumption) There exists a set of anchor state-action pairs $\mathcal{K}$ such that  for any
		$(s,a)\in\mathcal{S}\times\mathcal{A}$, its feature vector can be represented as a convex combination of the anchors $\{(s_k,a_k)|k\in\mathcal{K}\}$:
		$$\exists\{\lambda_k^{s,a}\}:\phi(s,a)=\sum_{k\in\mathcal{K}}\lambda_k^{s,a}\phi(s_k.a_k),\ \sum_{k\in\mathcal{K}}\lambda_k^{s,a}=1,\lambda_k\geq0,\forall k\in\mathcal{K},(s,a)\in(\mathcal{S},\mathcal{A}).$$
		\label{assump1}
	\end{assumption}
	
	Anchor-state assumption means the convex hull of feature vectors is a polyhedron with $|\mathcal{K}|$ nodes. Choosing these nodes to be the anchor-state set $\mathcal{K}$, the estimate given by Algorithm \ref{alg1} is guaranteed to be a probability matrix, as a direct application of Proposition \ref{prop1}. This assumption has also been studied as soft state aggregation model \citep{singh1995reinforcement,duan2019state}. Without loss of generality, we assume $\phi(s,a)$ to be a probability vector in the analysis, otherwise we can use $\{\lambda_k^{s,a}\}$ as the feature.
	The same anchor condition has been studied in \citep{Yang2019} with a model-free algorithm. 
	A modification of this condition has been proposed in \citep{zanette2019limiting} to reach linear error propagation for value iteration.  
	
	
	\begin{proposition}
		If we have $N$ i.i.d. samples from each state-action pair in $\mathcal{K}$, then an unbiased estimate $\widehat{P}$ of the transition model is obtained from Algorithm \ref{alg1}:
		$$\widehat{P}(s'|s,a)=\sum_{k\in\mathcal{K}}\lambda_k^{s,a}\widehat{P}_\mathcal{K}(s'|s_k,a_k),$$
		where $\widehat{P}_\mathcal{K}(s'|s,a)=\frac{\mathrm{count}(s,a,s')}{N}$ and $\{\lambda_k^{s,a}\}$ are coefficients defined in Proposition \ref{prop1}. In addition, if Assumption \ref{assump1} holds, $\widehat{P}$ is a probability transition matrix.
		\label{prop2}
	\end{proposition}
	
	Proposition \ref{prop2} shows that an empirical MDP $\widehat{M}=(\mathcal{S},\mathcal{A},\widehat{P},r,\gamma)$ can be constructed by substituting the unknown transition matrix $P$ in $\mathcal{M}$ with $\widehat{P}$. The reward function $r$ is assumed to be known.\footnote{If $r$ is unknown, we can assume it has the same structure as the transition matrix $p$ and estimate it in the same manner. With simple modification of our proof, we can show that only $O(K/\epsilon^2(1-\gamma)^2)$ is needed to get a sufficiently accurate estimation of $r$.} In the following sections, we will analyze the property of $\widehat{M}$ and prove the concentration property of the optimal policy in $\widehat{M}$. We will use $\widehat{V}$ and $\widehat{\pi}$ to denote value function and policy in $\widehat{M}$. 
	
	\begin{algorithm}[H]
		\SetAlgoLined
		\KwIn{A generative model that can output samples from distribution $P(\cdot|s,a)$ for query $(s,a)$, a plug-in solver.}
		\textbf{Initial}: Sample size: $N$, state-action set $\mathcal{K}$\;
		\For{(s,a) in $\mathcal{K}$}{
			Collect $N$ samples from $P(\cdot|s,a)$\;
		}
		Compute $\widehat{P}_\mathcal{K}(s'|s,a)=\frac{count(s,a,s')}{N}$\;
		Compute the linear combination coefficients $\lambda_{k}^{s,a}$ that satisfies $\phi(s,a)=\sum_{k\in\mathcal{K}}\lambda_k^{s,a}\phi(s_k.a_k)$\;
		Estimate transition distribution $\widehat{P}(s'|s,a)=\sum_{k\in\mathcal{K}}\lambda_k^{s,a}\widehat{P}_\mathcal{K}(s'|s_k,a_k)$\;
		Construct the empirical MDP: $\widehat{M}=(\mathcal{S},\mathcal{A},\widehat{P},r,\gamma)$ for DMDP, 
		$\widehat{M}=(\mathcal{S},\mathcal{A},\widehat{P},r,H)$ for FHMDP, or 
		$\widehat{M}=(\mathcal{S}_1.\mathcal{S}_2,\mathcal{A},\widehat{P},r,\gamma)$ for 2-TBSG\;
		Call the plug-in solver: input empirical model $\widehat{M}$ and output an $\epsilon_{\mathrm{PS}}$-optimal policy $\widehat{\pi}$ in $\widehat{M}$\;
		\KwOut{$\widehat{\pi}$}
		\caption{Plug-in Solver Based Reinforcement Learning}
		\label{alg1}
	\end{algorithm}

	\section{Plug-in Solver Approach for Linear Transition Models}  
	In this section, we analyze the sample complexity upper bounds for discounted MDP. Generally, we use Algorithm \ref{alg1} to construct an empirical MDP $\widehat{M}$ and then use a plug-in solver to find an $\epsilon_{\mathrm{PS}}$-optimal policy $\widehat{\pi}$ in $\widehat{M}$. Algorithm \ref{alg1} gives a formal framework of the plug-in solver approach. The target of our analysis is to prove that $\widehat{\pi}$ is an approximately optimal policy in the true MDP $M$.
	
	Intuitively, $\widehat{P}$ is close to $P$ when it is constructed with sufficiently many samples, so $\widehat{M}$ is similar to $M$ and the optimal policy in $\widehat{M}$ is an approximately optimal policy in $M$. However, 
	if we require that $\widehat{P}$ is close to $P$ in total variation distance,  which is indeed a sufficient condition to obtain an approximately optimal policy for $M$, the number of samples needed is proportional to $K|\mathcal{S}|$, and hence is sample-inefficient.
	The same phenomena has been affecting the sample complexity in the tabular setting, 
	\citep{azar2013minimax},
	where their sample complexity of getting a constant optimal policy is
	at least $O(|\mathcal{S}|^2|\mathcal{A}|(1-\gamma)^{-3})$, which is at least the number of entries of the probability transition matrix. 
	In the remaining part of this section, we leverage the sample de-coupling ideas from \citep{agarwal2019optimality},  variance preserving ideas from \citep{Yang2019}, and novel ideas to decouple MDP with linear feature representation to eventually establish our near-optimal sample complexity bound. 

	\subsection{Linear Transition Model with Anchor State Assumption}
	
	In this section, we gives a minimax sample complexity of Algorithm \ref{alg1} for feature-based MDP with anchor state assumption. The main results are shown below and then we introduce the auxiliary MDP technique, which is the key in the analysis. 
	
	\begin{theorem}(Sample complexity for DMDP)
		Suppose Assumption \ref{assump1} is satisfied and the empirical model $\widehat{M}$ is constructed as in Algorithm \ref{alg1}. Set $\delta\in(0,1)$ and $\epsilon\in(0,(1-\gamma)^{-1/2}]$. Let $\widehat{\pi}$ be an $\epsilon_{\mathrm{PS}}$-optimal policy in $\widehat{M}$.
		If $N\geq\frac{c\log(cK(1-\gamma)^{-1}\delta^{-1})}{(1-\gamma)^3\epsilon^2},$
		then with probability larger than $1-\delta$, we have
		\[Q^{\widehat{\pi}}\geq Q^*-\epsilon-\frac{3\epsilon_{\mathrm{PS}}}{1-\gamma},\]
		where $c$ is a constant.
		\label{thm1}
	\end{theorem}
	
	Theorem \ref{thm1} shows that with $KN=\widetilde{O}(K/\epsilon^2(1-\gamma)^3)$ samples, an $\epsilon_{\mathrm{PS}}$-optimal policy in $\widehat{M}$ is an $\epsilon+3\epsilon_{\mathrm{PS}}/(1-\gamma)$ policy in the true model with large probability. As we can solve $\widehat{M}$ to arbitrary accuracy (i.e. $\epsilon_{\mathrm{PS}}\rightarrow0$) without collecting additional samples, this sample complexity matches the sample complexity lower bound given in \citep{Yang2019}. We prove this theorem by using the auxiliary MDP technique to analyze the concentration property of $Q^{\widehat{\pi}^*}-\widehat{Q}^*$ and thus show that three terms in $Q^*-Q^{\widehat{\pi}}= (Q^*-\widehat{Q}^{\pi^*})+(\widehat{Q}^{\pi^*}-\widehat{Q}^*)+(\widehat{Q}^*-Q^{\widehat{\pi}})$ can be bounded. As we have $Q^\pi-\widehat{Q}^\pi=(I-\gamma P^\pi)^{-1}(P-\widehat{P})\widehat{V}^\pi$, which will be proved in the supplementary material, the main task in the analysis is to portray the concentration of $|(P-\widehat{P})\widehat{V}^{\pi^*}|$ and $|(P-\widehat{P})\widehat{V}^{*}|$. Due to the dependence between $\widehat{V}^{\pi^*}$, $\widehat{V}^{*}$ and $\widehat{P}$, conventional concentration arguments are not applicable. To decouple the dependence, we construct a series of auxiliary MDPs. In auxiliary models, transition distributions from all state-action pairs in $\mathcal{K}$ except a specific pair $(s,a)$ are equal to $\widehat{P}_\mathcal{K}$, while transition distribution from $(s,a)$ is $P(s,a)$. Then we prove that value function in $\widehat{M}$ can be approximated by tuning the reward in auxiliary model. Now, we give a rigorous definition for auxiliary MDP.
	
	\begin{definition}
		(Auxiliary DMDP) Suppose Assumption \ref{assump1} holds. For an empirical MDP $\widehat{\mathcal{M}}=(\mathcal{S},\mathcal{A},\widehat{P}=\Phi\widehat{P}_K,r,\gamma)$ and a given state
		pair $(s,a)\in\mathcal{K}$, the auxiliary transition model is 
		$\widetilde{\mathcal{M}}_{s,a,u}=(\mathcal{S},\mathcal{A},
		\widetilde{P}=\Phi
		\widetilde{P}_K,r+u\Phi^{s,a},\gamma)$, where
		$$\widetilde{P}_\mathcal{K}(s',a')=
		\begin{cases}
		\widehat{P}_\mathcal{K}(s',a')& if\ (s',a')\neq(s,a),\\
		P(s,a)& otherwise.
		\end{cases}$$
		$\Phi^{s,a}\in\mathbb{R}^{|\mathcal{S}||\mathcal{A}|}$ is the column vector of $\Phi$ that corresponds to $(s,a)$ in $\mathcal{K}$ and $u\in\mathbb{R}$ is a scalar.
	\end{definition}
	
	As $\Phi$ is a probability transition matrix by Assumption \ref{assump1}, $\widetilde{P}=\Phi\widetilde{P}_\mathcal{K}$ is a probability transition matrix as well, so the auxiliary MDP is well defined. We use $\widetilde{V}_{s,a,u}$, $\widetilde{Q}_{s,a,u}$ and $\widetilde{\pi}_{s,a,u}$ to denote value function, action-value function and policy in $\widetilde{M}_{s,a,u}$, and we omit $(s,a)$ when there is no misunderstanding. We prove that there always exist a $u$ such that $\widehat{Q}^\pi=\widetilde{Q}_{s,a,u}^\pi$.
	
	If we take $\widehat{Q}^\pi$ as a function of $\widehat{P}(s,a)$, then this function maps a $|\mathcal{S}|-1$ dimensional probability simplex to $\mathbb{R}^{|\mathcal{S}||\mathcal{A}|}$. We observe a surprising fact that the range of $\widehat{Q}^\pi$ actually lies in a one dimensional manifold in $\mathbb{R}^{|\mathcal{S}||\mathcal{A}|}$, as tuning the scalar $u$ is sufficient to recover $\widehat{Q}^\pi$. Therefore, a one-dimensional $\epsilon$-net is sufficient to capture $\widehat{Q}^\pi$. By using evenly spaced points in the interval $U^\pi_{s,a}$, which is a bounded interval that contains all possible $u$, we show that the Q-function of corresponding auxiliary MDPs , $\widetilde{Q}_{s,a,u}$, forms the $\epsilon$-net. We prove that the auxiliary MDP is robust to coefficient $u$, and that an $\epsilon$-net on $u$ leads to an $\epsilon/(1-\gamma)$-net on $\widetilde{Q}_{s,a,u}$. Given these properties of auxiliary DMDP, we can provide the proof sketch of Theorem \ref{thm1}. We denote a finite set $B^\pi_{s,a}$ to be made up of enough evenly spaced points in interval $U^\pi_{s,a}$. If $|B^\pi_{s,a}|$ is large enough, $B^\pi_{s,a}$ contains a $u'$ that is close enough to $u^\pi$. Thus $\widetilde{Q}_{u'}^\pi$ approximates $\widehat{Q}^\pi$. As $B^\pi_{s,a}$ has no dependence on $\widehat{P}(s,a)$, a union bound can be provided with Beinstein inequality and total variance technique \citep{azar2013minimax,Yang2019, agarwal2019optimality}. Combining everything together, we can concentrate $|(P(s,a)-\widehat{P}(s,a))\widehat{V}^\pi|$ and hence finally bound $Q^*-Q^{\widehat{\pi}}$.

	The next result is about the sample complexity when model misspecification exists. Model misspecification means that the transition matrix cannot be fully expressed by the linear combination of features, i.e. $P=\Phi\Bar{P}_\mathcal{K}+\Xi$, where $\Xi$ is the approximation error or noise of linear transition model and $\Phi\Bar{P}_\mathcal{K}$ is the underlying true transition. In this condition, the estimate $\widehat{P}$ given by Algorithm \ref{alg1} can be biased, and the degree of perturbation depends on $\xi=\|\Xi\|_\infty$.
	
	\begin{theorem}(Sample complexity for DMDP with model misspecification)
		Suppose $M=(\mathcal{S},\mathcal{A},P,r,\gamma)$ has an approximate linear transition model $\Bar{P}$ such that $\Bar{P}$ admits a linear feature representation $\phi$ and there exists some $\xi\geq0$ that $\|P(\cdot|s,a)-\Bar{P}(\cdot|s,a)\|_{TV}\leq \xi,\forall(s,a)\in(\mathcal{S},\mathcal{A}).$ Suppose Assumption \ref{assump1} is satisfied and the empirical model $\widehat{M}$ is constructed as in Algorithm \ref{alg1}. Set $\delta\in(0,1)$ and $\epsilon\in(0,(1-\gamma)^{-1/2}]$. Let $\widehat{\pi}$ be an $\epsilon_{\mathrm{PS}}$-optimal policy for $\widehat{M}$.
		If $N\geq\frac{c\log(cK(1-\gamma)^{-1}\delta^{-1})}{(1-\gamma)^3\epsilon^2},$
		then with probability larger than $1-\delta$,
		$$Q^{\widehat{\pi}}\geq Q^*-\epsilon-\frac{3\epsilon_\mathrm{PS}}{1-\gamma}-\frac{16\sqrt{\xi}}{(1-\gamma)^2},$$
		where $c$ is a constant. 
		\label{thm2}
	\end{theorem}
	
	Theorem \ref{thm2} implies that for an approximately linear transition model, the suboptimality of $\widehat{\pi}$ only increase $\frac{16\sqrt{\xi}}{(1-\gamma)^2}$. The proof is given in the supplementary material and the high level idea is to separate the linear model part and perturbation part in the empirical model.

	

	\subsection{General Linear Transition Model}
	
	If the anchor-state assumption is not satisfied, the estimate $\widehat{P}$ given by Algorithm \ref{alg1} is not guaranteed to be a non-negative matrix, and thus not a probability transition matrix. Next proposition shows that even if the features approximately satisfy the anchor state assumption, the estimated transition $\widehat{P}$ is still not a probability transition matrix with high probability.
	
	\begin{proposition}
		Suppose the samples are from state-action pairs $\mathcal{K}$ and the unbiased estimate of the transition probability matrix given by Algorithm \ref{alg1} is $\widehat{P}$. If $\max_{s,a}\sum_{k\in\mathcal{K}}|\lambda_k^{s,a}|=L>1$ and $K\geq2$, where $\{\lambda_k^{s,a}\}$ is defined as in Proposition \ref{prop1}, there exists true model $P$ such that with probability larger than $1/3$, $\widehat{P}$ is not a probability transition matrix for sample size $N\geq C$, where $C$ is a constant.
		\label{prop5}
	\end{proposition}
	
	Note that we have $\max_{s,a}\sum_{k\in\mathcal{K}}\lambda_k^{s,a}=1$, so $L\geq1$ is always satisfied and $L=1$ is equivalent to anchor state assumption. Proposition \ref{prop5} shows that without anchor state assumption, we can always find an MDP $M$ such that the empirical MDP $\widehat{M}$ is not a well-defined MDP with large probability. The estimated transition distributions $\widehat{P}$ sum to one, but are not necessarily non-negative. This kind of MDP is known as pseudo MDP~\citep{yao2014pseudo} and \citep{lever2016compressed,pires2016policy,lattimore2019learning} analyze the error bound induced under unrealizability. In a pseudo MDP, well-known properties of bellman operator like contraction and monotonicity no longer exist, so there is no optimal value/policy. An example of pseudo MDP without optimal value/policy is provided in the supplementary material. However, plug-in solvers like value iteration still work in pseudo MDP. In this part, we give a sample complexity bound for value iteration solver. To facilitate analysis, we need a regularity assumption on features.
	
	
	\begin{assumption}
		(Representative States and Regularity of Features) There exists a set of state-action pairs $\mathcal{K}$ and a scalar $L\geq1$ such that for any
		$(s,a)\in\mathcal{S}\times\mathcal{A}$, its feature vector can be represented as a linear combination of the state-actions $\{(s_k,a_k)|k\in\mathcal{K}\}$:
		$$\exists\{\lambda_k^{s,a}\}:\phi(s,a)=\sum_{k\in\mathcal{K}}\lambda_k^{s,a}\phi(s_k.a_k), \sum_{k\in\mathcal{K}}|\lambda_k^{s,a}|\leq L,\forall(s,a)\in(\mathcal{S},\mathcal{A}).$$
		\label{assump2}
	\end{assumption}
	
	This assumption means that the selected state-action pairs $\mathcal{K}$ can represent all features by linear combinations with bounded coefficients, which avoids error explosion in the iterative algorithm. In theorem \ref{thm3}, we give the sample complexity for value iteration solver and the proof is provided in the supplementary material.
	
	\begin{theorem}(Value iteration solver for general linear MDP)
		Suppose Assumption \ref{assump2} is satisfied and the empirical model $\widehat{M}$ is constructed as in Algorithm \ref{alg1}. Set $\delta\in(0,1)$ and $\epsilon\in(0,1/(1-\gamma))$. $\widehat{V}$ is the value function of applying value iteration for $O(1/(1-\gamma))$ times in the empirical MDP $\widehat{M}=(\mathcal{S},\mathcal{A},\widehat{P},r,\gamma)$. Let $\widehat{\pi}$ be the greedy policy with respect to $\widehat{V}$.
		If $N\geq\frac{cL^2\log(cK(1-\gamma)^{-1}\delta^{-1})}{\epsilon^2}\mathrm{poly}(1/(1-\gamma)),$
		then with probability larger than $1-\delta$,
		$$Q^{\widehat{\pi}}\geq Q^*-\epsilon,$$
		where $c$ is a constant.
		\label{thm3}
	\end{theorem}
	
	This theorem shows that without anchor-state assumption, value iteration solver is still a sample efficient algorithm. For pseudo MDP, the optimal policy/value no longer exist, so different plug-in solver may output significantly different policies. Even though we can assume all solvers are time efficient,  an one to one analysis is needed and we leave this as future work. 
	
	\section{Extensions to Finite Horizon MDP and Two-players Turn-based Stochastic Game}
	
	Our approach for discounted MDP can be extended to finite horizon MDP (FHMDP) and two-players turn-based stochastic game (2-TBSG). Due to space limitation, preliminary about FHMDP and 2-TBSG is given in the supplementary material. All of the three decision models use a transition probability matrix to represent the system dynamic, so the framework of plug-in solver approach is the same (i.e. Algorithm \ref{alg1}). Here we directly present the sample complexity results here.
	
	\begin{theorem}(Sample complexity for FHMDP)
		Suppose Assumption \ref{assump1} is satisfied and the empirical model $\widehat{M}$ is constructed as in Algorithm \ref{alg1}. Set $\delta\in(0,1)$ and $\epsilon\in(0,(1-\gamma)^{-1/2}]$. Let $\widehat{\pi}$ be any $\epsilon_{\mathrm{PS}}$-optimal policy for $\widehat{\mathcal{M}}$.
		If $N\geq\frac{c\log(cKH\delta^{-1})H^3\min\{|\mathcal{S}|,K,H\}}{\epsilon^2},$
		then with probability larger than $1-\delta$,
		$Q^{\widehat{\pi}}\geq Q^*-\epsilon-3\epsilon_{\mathrm{PS}}H$,
		where $c$ is a constant. 
	\end{theorem}
	
	\begin{theorem}(Sample complexity for 2-TBSG)
		Suppose Assumption \ref{assump1} is satisfied and the empirical model $\widehat{M}$ is constructed as in Algorithm \ref{alg1}. Set $\delta\in(0,1)$ and $\epsilon\in(0,(1-\gamma)^{-1/2}]$. Let $\widehat{\pi}=(\widehat{\pi}_1,\widehat{\pi}_2)$ be any $\epsilon_{\mathrm{PS}}$-optimal policy for $\widehat{\mathcal{M}}$.
		If $N\geq\frac{c\log(cK(1-\gamma)^{-1}\delta^{-1})}{(1-\gamma)^3\epsilon^2},$
		then with probability larger than $1-\delta$,
		$|Q^{\widehat{\pi}}-Q^*|\leq \epsilon+\frac{3\epsilon_\mathrm{PS}}{1-\gamma}$,
		where $c$ is a constant.
	\end{theorem}
	
	The above theorems indicate that plug-in solver is sample efficient for FHMDP and 2-TBSG. The high level idea of the proof is similar to the analysis of discounted MDP, i.e. we use a series of auxiliary MDP to break the statistical dependence. For FHMDP, the rewards in auxiliary MDPs need to be tuned step by step, which leads to the extra $H$ dependency. For 2-TBSG, we analyze the counter policy, or known as the best response policy, in true model and in empirical model. The details are problem-specific and due to space limitation, we provide them in the supplementary material.
	
	\section{Discussion}
	
	This paper studies the sample complexity of plug-in solver approach in feature-based MDPs, including discounted MDP, finite horizon MDP and stochastic games. We tackle a basic and important problem in reinforcement learning: whether planning in an empirical model is sample efficient to give an approximately optimal policy in the real model. To our best knowledge, this is the first result proving  minimax sample complexity for the plug-in solver approach in feature-based MDPs. We hope that the new technique in our work can be reused in more general settings and motivate breakthroughs in other domains.
	
	Our work also opens up several directions for future research, which is listed below.
	
	\begin{itemize}
		\item \textit{Improve the dependence on $H$ for finite horizon MDP.} The sample complexity we give for finite horizon MDP is $\widetilde{O}(KH^4\epsilon^{-2})$, which has an extra $H$ dependence compared with the discounted case. We conjecture that the plug-in solver approach should enjoy the optimal $\widetilde{O}(KH^3\epsilon^{-2})$ complexity as model-free algorithms \citep{Yang2019}. This may require new techniques as absorbing MDP is not well suitable for finite horizon MDP.
		
		\item \textit{Improve the result for stochastic game.} Our result for turn-based stochastic game is for finding a strategy with value $\epsilon$-close to the Nash equilibrium value, while the final objective is to find an $\epsilon$-Nash equilibrium value (for definition see \citep{Sidford2019}). We need a more refined analysis to tackle the complex dependence between two players.
		
		\item \textit{Extend the range of $\epsilon$.} Our result holds for $\epsilon\in(0,\frac{1}{\sqrt{1-\gamma}}]$, which is better than previous model-free result \citep{Yang2019}. Recently, \citep{li2020breaking} develops a novel reward perturbation technique to prove that the minimax sample complexity holds for $\epsilon\in(0,\frac{1}{1-\gamma}]$. A direct application of their result would lead to $\log(|\mathcal{S}||\mathcal{A}|)$ factor in linear MDP. The full $\epsilon$ range for linear MDP is still an open problem.
		
		\item \textit{Beyond linear MDP.} Currently most works focus on tabular and linear setting with generative model, which provide intuition for solving the ultimate setting of general function approximation like neural network. One interesting problem is whether we can develop provably efficient model-based algorithm under general function approximation setting, as the construction of the empirical model seems to be difficult even for linear Q-function assumption. 
	\end{itemize}

	\clearpage
	
	\section*{Broader Impact}
	We believe that our work can benefit both theory and algorithm researches of model-based reinforcement learning (MBRL), as the sample efficiency of MBRL has been long observed but lack theoretical analysis. Previously few sample complexity results for MBRL in feature-based MDP exist, so people may wonder if the gap between theory and application is treatable. Our work answers this question positively by theoretically proving the simplest MBRL method, which is the plug-in solver approach, can reach the minimax sample complexity. Researchers that are interested in MBRL theory can benefit from our results and techniques. In addition, we give a thorough analysis on linear features and show that anchor-state condition can measure the quality of features. This result can guide reinforcement learning practioners to design features and sample efficient algorithms. 
	We believe that our research has no negative societal effects.

\section*{Acknowledgements}
We thank all anonymous reviewers for their insightful comments.

\medskip

\small

\bibliography{nips2020}
\bibliographystyle{plainnat}

\appendix

\paragraph{Notations} We use $P=\Phi\Psi$ to represent the linear transition model, where $\Phi\in\mathbb{R}^{|\mathcal{S}||\mathcal{A}|\times K}$ is the feature matrix consisting of feature vectors $\phi(s,a)$ and $\Psi$ is the corresponding unknown matrix from $\psi$. Proposition 1 means that $\Phi$ can be factorized as $\Phi=\Lambda\Phi_\mathcal{K}$ where $\Lambda\in\mathbb{R}^{|\mathcal{S}||\mathcal{A}|\times K}$ is the matrix of $\lambda_{k}^{s,a}$ and $\Phi_\mathcal{K}$ is the submatrix of $\Phi$. Thus we have $P=\Lambda\Phi_\mathcal{K}\Psi=\Lambda P_\mathcal{K}$. In addition, if we have feature $\Phi$ and state-action pairs $\mathcal{K}$, we can compute $\Lambda$ and $\Lambda$ can also be regarded as feature. Thus, we will use $\Phi$ and $\phi$ to represent $\Lambda$ and $\lambda$ in Appendix C and Appendix D. 

We use $P^\pi\in\mathbb{R}^{|\mathcal{S}||\mathcal{A}|\times|\mathcal{S}||\mathcal{A}|}$ to denote the transition distributions on state action pairs that are induced by policy $\pi$. $P^*\in\mathbb{R}^{|\mathcal{S}||\mathcal{A}|\times|\mathcal{S}||\mathcal{A}|}$ is the transition matrix induced by optimal policy $\pi^*$. $P_\mathcal{K}\in\mathbb{R}^{|\mathcal{K}|\times|\mathcal{S}|}$ is the submatrix of $P$ consisting of rows in $\mathcal{K}$. $\widehat{P}^\pi,\widehat{P}^*,\widehat{P}_\mathcal{K},\Bar{P}_\mathcal{K},\widetilde{P}^\pi,\widetilde{P}^*,\widetilde{P}_\mathcal{K}$ 
are defined in a similar manner. We use $P(s,a)$ to denote the row vector of $P$ that corresponds to $(s,a)$.

For a vector $V$, we use $V^2$, $|V|$, $\sqrt{V}$ and $<$ to denote the component-wise square, absolute value, square root and less-than. We define $Var_{s,a}(V):=P(s,a)V^2-(P(s,a)V)^2$, $Var_\mathcal{K}(V)=[Var_{s_1,a_1}(V),\cdots,Var_{s_K,a_K}]\in\mathbb{R}^K$ and $Var_P(V)\in\mathbb{R}^{|\mathcal{S}||\mathcal{A}|}$ to  consist of all $Var_{s,a}(V)$. A detailed description is provided in \cite{azar2013minimax}. We use $\mathbf{1}$ to denote a column vector with all components to be 1. We use $[H]$ to denote $\{0,1,\cdots,H-1\}$.

\section{Additional Preliminary}

\paragraph{Finite Horizon Markov Decision Process}
A Finite Horizon Markov decision process (FHMDP) is described by the tuple $\mathcal{M}=(\mathcal{S},\mathcal{A},P,r,H)$, which differs from DMDP only in that the discount factor $\gamma$ is replaced by the horizon $H$. The value function of policy $\pi$ is defined as:
$$V_h^\pi(s):=\mathbb{E}\left[\sum_{t=h}^{H-1}r(s^t,\pi_t(s^t))|s^h=s\right],$$
which depends on state $s$ and time step $h$. The target of the agent is to learn a policy  $\pi=(\pi_0,\cdots,\pi_{H-1}),\pi_h:\mathcal{S}\rightarrow\mathcal{A},\forall h\in[H]$ that maximize the total reward $V_0^\pi(s)$ from any initial state $s$.
It is well known that the optimal policy for FHMDP maximizes the value function in each time step:
$$V^*_h(s):=V^{\pi^*}_h(s)=\max_\pi V^\pi(s),\forall s\in\mathcal{S},h\in[H].$$
The action-value or Q-function $Q^\pi_h$ and optimal Q-function $Q^*$ are defined similarly as in DMDP. The relation between Q-function and value function is:
$$V_h^\pi(s)=Q_h^\pi(s,\pi_h(s)),
\quad
Q_h^\pi(s,a)=r(s,a)+\gamma P(s,a)V^\pi_{h+1}.$$

If the value function of a policy $\pi$ is $\epsilon$-close to the optimal value function in all time steps, the policy is called a $\epsilon$-optimal policy:
$$Q_h^*(s,a)\geq Q_h^\pi(s,a)\geq Q_h^*(s,a)-\epsilon,\quad \forall (s,a),h\in[H].$$

Without loss of generality, we assume $r\in[0,1]$ and then we have $V_h^\pi\in[0,H-h]$. Note that in a normal FHMDP, the reward in each time step is identical, but in our analysis, we will construct FHMDPs with different rewards in each step.

\paragraph{Two Player Turn-based Stochastic Game}

A discounted turn-based two-player zero-sum stochastic games (2-TBSG) is described as the tuple $\mathcal{G}=(\mathcal{S}_{max},\mathcal{S}_{min},\mathcal{A},P,r,\gamma)$. It is a generalized version of DMDP which includes two players competing with each other. Player 1 aims to maximize the total reward with policy $\pi_1$ while player 2 aims to minimize it with policy $\pi_2$. We denote policy $\pi:=(\pi_1,\pi_2)$ to be the overall policy. Given policy $\pi$, the value function and Q-function can be defined as in DMDP.

From the perspective of player 1, if the policy of player 2 $\pi_2$ is given, the 2-TBSG degenerates to a DMDP, so the optimal policy exists for player 1. This optimal policy depends on $\pi_2$, so we call it the counter policy to $\pi_2$ and denote it as $c_{1}(\pi_2)$. Similarly we can define $c_2(\pi_1)$ as the counter policy of $\pi_1$ for player 2. For simplicity, we ignore the subscript in $c_1$ and $c_2$ when it is clear in the context.

To solve a 2-TBSG, our goal is to find the Nash equilibrium policy $\pi^*=(\pi_1^*,\pi_2^*)$, where $\pi_1^*=c(\pi_2^*),\pi_2^*=c(\pi_1^*)$. For this policy, neither player can benefit from change its policy alone.

We give the following well-known properties of 2-TBSG without proof (see. e.g. \cite{Sidford2019}), which can be regarded as generalized optimality property in DMDP.
\begin{itemize}
	\item $V^{c(\pi_2),\pi_2}=\max_{\pi_1} V^{\pi_1,\pi_2}$
	\item $V^{\pi_1,c(\pi_1)}=\max_{\pi_2} V^{\pi_1,\pi_2}$
	\item $V^{\pi_1^*,\pi_2^*}=\min_{\pi_2}V^{c(\pi_2),\pi_2}$
    \item $V^{\pi_1^*,\pi_2^*}=\max_{\pi_1}V^{\pi_1,c(\pi_1)}$
\end{itemize} 
Our target is to find an $\epsilon$-optimal policy $\pi=(\pi_1,\pi_2)$ such that $\big|Q^\pi(s,a)-Q^*(s,a)\big|\leq\epsilon,\forall (s,a)$ for some $\epsilon>0$.

\section{Empirical Model Construction}

In this section, we prove Proposition 1, Proposition 2 and Proposition 3, which give the theoretical guarantees of our model construction algorithm (Algorithm 1). 

\begin{proof}[Proof of Proposition 1]
	Here we prove the three arguments in Proposition 1. 
	\begin{enumerate}
		\item As $\widehat{P}_\mathcal{K}(s'|s_k,a_k)=\frac{\mathrm{count}(s_k,a_k)}{N}$, $\widehat{P}_\mathcal{K}$ is an unbiased estimate of $P_\mathcal{K}$. Hence we have \[\mathbb{E}\widehat{P}(s'|s,a)=\mathbb{E}\sum_{k\in\mathcal{K}}\lambda_k^{s,a}\widehat{P}_\mathcal{K}(s'|s_k,a_k)=\sum_{k\in\mathcal{K}}\lambda_k^{s,a}P(s'|s_k,a_k)=P(s'|s,a),\]
		where the last equality is from $P=\Lambda P_\mathcal{K}$.
		\item We have that 
		\begin{align*}
		    \sum_{k\in\mathcal{K}}\lambda_k^{s,a}&=\sum_{k\in\mathcal{K}}\sum_{s'\in\mathcal{S}}\lambda_k^{s,a}P(s'|s_k,a_k)\\
		    &=\sum_{s'\in\mathcal{S}}\sum_{k\in\mathcal{K}}\lambda_k^{s,a}P(s'|s_k,a_k)\\
		    &=\sum_{s'\in\mathcal{S}}P(s'|s,a)\\
		    &=1.
		\end{align*}
		Thus we have $\sum_{k\in\mathcal{K}}\lambda_k^{s,a}=1$. Then we have
		\begin{align*}
		    \sum_{s'\in\mathcal{S}}\widehat{P}(s'|s,a)&=\sum_{s'\in\mathcal{S}}\sum_{k\in\mathcal{K}}\lambda_k^{s,a}\widehat{P}_\mathcal{K}(s'|s_k,a_k)\\
		    &=\sum_{k\in\mathcal{K}}\sum_{s'\in\mathcal{S}}\lambda_k^{s,a}\widehat{P}_\mathcal{K}(s'|s_k,a_k)\\
		    &=\sum_{k\in\mathcal{K}}\lambda_k^{s,a}\\
		    &=1.
		\end{align*}

		\item
		If $\lambda_k^{s,a}\geq0,\forall k\in\mathcal{K},(s,a)\in(\mathcal{S},\mathcal{A})$, then we have
		\[\widehat{P}(s'|s,a)=\sum_{k\in\mathcal{K}}\lambda_k^{s,a}\widehat{P}_\mathcal{K}(s'|s_k,a_k)\geq0,\] as every component is non-negative. 
	\end{enumerate} 
\end{proof}

\begin{remark}
    Proposition 1 implies that when anchor-state condition is satisfied, $\Lambda$ is a probability matrix. Thus $P=\Lambda P_\mathcal{K}$ is factorized as a probability matrix into two probability transition matrix.
\end{remark}

\begin{proof}[Proof of Proposition 2]
	Combining Proposition 1 and Assumption 1, we directly have that $\sum_{s'\in\mathcal{S}}\widehat{P}(s'|s,a)=1,\widehat{P}(s'|s,a)\geq0$ for all $(s',s,a)$. Therefore $\widehat{P}$ is an eligible transition kernel. With Assumption 1 or not, we always have $\mathbb{E}\widehat{P}(s'|s,a)=P(s'|s,a)$, which means $\widehat{P}$ is an unbiased estimate of $P$.
\end{proof}

\begin{proof}[Proof of Proposition 3]
	We consider the following case. Suppose $\mathcal{K}=\{(s_1,a_1),\cdots,(s_K,a_K)\}$. For a specific $(s,a)\notin\mathcal{K}$, we have $\lambda_{s_1,a_1}^{s,a}=\frac{1+L}{2}$, $\lambda_{s_2,a_2}^{s,a}=\frac{1-L}{2}$ and $\lambda_{s_k,a_k}^{s,a}=0$ for $k\neq1,2$. For all other $(s',a')\neq(s,a)$, we set
	$$\lambda_{s'',a''}^{s',a'}=\begin{cases}
	\mathbf{1}_{\{(s'',a'')=(s',a')\}}&\mathrm{if} (s',a')\in\mathcal{K}\\
	\frac{1}{K}&\mathrm{otherwise}
	\end{cases}$$
It is easy to check that $\{\lambda_k^{s,a}\}$ is valid (We can set $\Phi=\Lambda$ and with state-action set $\mathcal{K}$, then the corresponding linear combination coefficients are $\Lambda$ as $\Phi_\mathcal{K}=I$).
	We set $P(s_1|s_1,a_1)=\frac{L-1}{L+1},P(s_2|s_1,a_1)=\frac{2}{L+1}$ and $P(s_1|s_2,a_2)=1$. Other transition distribution can be defined arbitrarily to construct an eligible $P$.
	
	As $\mathrm{count}(s_1,a_1,s_1)$ follows a binomial distribution with $p=P(s_1|s_1,a_1)=\frac{L-1}{L+1}$ and $\mathrm{count}(s_2,a_2,s_1)=N$ as $P(s_1|s_2,a_2)=1$, the estimate of $P(s_1|s,a)=\sum_{k\in\mathcal{K}}\lambda_{k}^{s,a}P(s_1|s_k,a_k)=0$ is
	\begin{align*}
	    \widehat{P}(s_1|s,a)&=\sum_{k\in\mathcal{K}}\lambda_{k}^{s,a}\widehat{P}(s_1|s_k,a_k)\\
	    &=\lambda_{s_1,a_1}^{s,a}\frac{\mathrm{count}(s_1,a_1,s_1)}{N}+\lambda_{s_2,a_2}^{s,a}\frac{\mathrm{count}(s_1,a_1,s_1)}{N}\\
	    &=\frac{L+1}{2}\frac{\mathrm{count}(s_1,a_1,s_1)}{N}-\frac{L-1}{2}.
	\end{align*}
	$\widehat{P}(s_1|s,a)$ is a translated and scaled binomial distribution with zero mean. By central-limit theorem, $\sqrt{N}\widehat{P}(s_1|s,a)$ converges to Gaussian distribution with zero mean. Therefore, there exists constant $C$ that if $N\geq C$, we have $\mathbb{P}(\widehat{P}(s_1|s,a)\leq0)\geq\frac{1}{3}$. This means the estimate $\widehat{P}$ is not a probability transition matrix with probability larger than 1/3.
\end{proof}

\begin{remark}
    This example can be generalized to prove that $\widehat{P}$ is not non-negative with probability larger than $1-1/3^{[K/2]}$, by choosing $[K/2]$ such $s,a$ pairs and each pair independently leads to an ineligible estimate with probability 1/3.
\end{remark}

Now we give an example of pseudo MDP to show that the optimality in normal MDP no longer exists, which means that there is no policy $\pi^*$ such that $V^{\pi^*}(s)=\max_{\pi}V^\pi(s)$ for all $s\in\mathcal{S}$. Note that the Bellman operator $\mathcal{T}^\pi[Q]=r+\gamma P^\pi Q$ may still be an contraction if $\|P^\pi\|_{1,\infty}\leq\frac{1}{\gamma}$, but the monotonicity no longer exist, which means we do not have $\mathcal{T}^\pi[Q_1]\geq\mathcal{T}^\pi[Q_2]$ if $Q_1\geq Q_2$.

We construct a pseudo MDP with $\mathcal{S}=\{s_1,s_2\},\mathcal{A}=\{a_1,a_2\}$. The transition distributions and rewards are 
$$P(s_1,a_1)=[0,1],P(s_1,a_2)=[0,1],P(s_2,a_1)=[1,0],P(s_2,a_2)=[-0.1,1.1],$$
$$r(s_1,a_1)=1,r(s_1,a_2)=0,r(s_2,a_1)=0,r(s_2,a_2)=1.$$
In this pseudo MDP, there are four policies $\pi_1,\pi_2,\pi_3,\pi_4$, which correspond to choosing action $(a_1,a_1),(a_1,a_2),(a_2,a_1),(a_2,a_2)$ in state $(s_1,s_2)$. We use $V_1,V_2,V_3,V_4\in\mathbb{R}^2$ to denote the corresponding value functions. Using Bellman equation, we have that
$$V_1=[\frac{1}{1-\gamma^2},\frac{\gamma}{1-\gamma^2}],V_2=[0,0],V_3=[\frac{1}{1-\gamma},\frac{1}{1-\gamma}],
$$
$$V_4=[\frac{\gamma}{0.1\gamma^2-1.1\gamma+1},\frac{1}{0.1\gamma^2-1.1\gamma+1}].$$
With some calculation, we have that
$$\mathrm{argmax}_iV_i(s_1)=2,\mathrm{argmax}_iV_i(s_2)=3,\mathrm{if}\ \ \gamma\leq\frac{10}{11}.$$ 
Therefore, in this pseudo MDP, no policy can achieve optimality in all states and no solver can output an $\epsilon$-optimal policy for arbitrary small $\epsilon$.

\section{Linear Transition Model with Anchor State Assumption}

\subsection{Sample Complexity for Discounted MDP}

Here we give the formal proof of Theorem 1. First we give the definition of $U_{s,a}^\pi$ and $U_{s,a}^*$, which is the set that contains all possible $u$ for different policies in auxiliary MDP. 

\begin{definition} (Feasible Set for $u$)
	For the auxiliary transition model $\widetilde{M}_{s,a,u}=(\mathcal{S},\mathcal{A},\widetilde{P}=\Phi\widetilde{P}_\mathcal{K},r+u\Phi^{s,a},\gamma$), $U^\pi_{s,a}$ is defined as the set of $u$ such that $\widetilde{V}^\pi_u\in[0,1/(1-\gamma)]^\mathcal{S}$, and $U^*_{s,a}$ is defined as the set of $u$ such that $\widetilde{V}^*_u\in[0,1/(1-\gamma)]^\mathcal{S}$.
\end{definition}

\begin{remark}
	Obviously, $u$ that satisfies $0\leq r+u\Phi^{s,a}\leq1$ is in both $U^\pi_{s,a}$ for arbitrary $\pi$ and $U^*_{s,a}$. Immediately, we have $0\in U^\pi_{s,a}$ for arbitrary $\pi$ and $0\in U^*_{s,a}$. Note that both $U^\pi_{s,a}$ and $U^*_{s,a}$ are independent of $\widehat{P}(s,a)$ and they are bounded intervals.
\end{remark}

\paragraph{Notations}
We use $\widetilde{V}_{s,a,u}$, $\widetilde{Q}_{s,a,u}$ and $\widetilde{\pi}_{s,a,u}$ to denote value function, action value function and policy in $\widetilde{M}_{s,a,u}$, and we omit (s,a) when there is no misunderstanding.

\begin{lemma}
	For an $\epsilon_\mathrm{PS}$-optimal policy $\widehat{\pi}$ in $\widehat{\mathcal{M}}$, we have
	$$\big\|Q^*-Q^{\widehat{\pi}}\big\|_\infty\leq\big\|Q^*-\widehat{Q}^{\pi^*}\big\|_\infty+\big\|\widehat{Q}^{\widehat{\pi}}-Q^{\widehat{\pi}}\big\|_\infty+\epsilon_\mathrm{PS}.$$
	\label{l1}
\end{lemma}

\begin{proof}
By the definition of $V^*$, we have
	\begin{align*}
	0<Q^*-Q^{\widehat{\pi}}&=Q^*-\widehat{Q}^{\pi^*}+\widehat{Q}^{\pi^*}-\widehat{Q}^*+\widehat{Q}^*-\widehat{Q}^{\widehat{\pi}}+\widehat{Q}^{\widehat{\pi}}-Q^{\widehat{\pi}}\\
	&\leq\big|Q^*-\widehat{Q}^{\pi^*}\big|+0+\epsilon_\mathrm{PS}+\big|\widehat{Q}^{\widehat{\pi}}-Q^{\widehat{\pi}}\big|.
	\end{align*}
	which implies $\big\|Q^*-Q^{\widehat{\pi}}\big\|_\infty\leq\big\|Q^*-\widehat{Q}^{\pi^*}\big\|_\infty+\big\|\widehat{Q}^{\widehat{\pi}}-Q^{\widehat{\pi}}\big\|_\infty+\epsilon_\mathrm{PS}.$
\end{proof}

\begin{lemma}
	We have
	$$Q^*-\widehat{Q}^{\pi^*}=(I-\gamma P^{\pi^*})^{-1}(\widehat{P}-P)\widehat{V}^{\pi^*},$$
	$$Q^{\widehat{\pi}}-\widehat{Q}^{\widehat{\pi}}=(I-\gamma P^{\widehat{\pi}})^{-1}(\widehat{P}-P)\widehat{V}^{\widehat{\pi}}.$$
	\label{l2}
\end{lemma}

\begin{proof}
	For any policy $\pi$,
	\begin{align*}
	Q^\pi-\widehat{Q}^\pi&=(I-\gamma P^\pi)^{-1}r-(I-\gamma\widehat{P}^\pi)^{-1}r\\
	&=(I-\gamma P^\pi)^{-1}((I-\gamma\widehat{P}^\pi)-(I-\gamma P^\pi))\widehat{Q}^\pi\\
	&=\gamma(I-\gamma P^\pi)^{-1}(P^\pi-\widehat{P}^\pi)\widehat{Q}^\pi\\
	&=\gamma(I-\gamma P^\pi)^{-1}(P-\widehat{P})\widehat{V}^\pi.
	\end{align*}
	Set $\pi=\pi^*$ and $\pi=\widehat{\pi}$, then we can get the results.
\end{proof}

\begin{lemma}
	For any value function $V$ and state action pair $(s,a)$,
	$$\phi(s,a)\sqrt{Var_{\mathcal{K}}(V)}=\sum_{k=1}^d\phi_k(s,a)\sqrt{Var_{s_k,a_k}(V)}\leq\sqrt{Var_{s,a}(V)}.$$
	\label{l3}
\end{lemma}

\begin{proof}
	Since $\phi(s,a)$ is a probability transition matrix, we can use Jensen's inequality here.
	\begin{align*}
	\phi(s,a)\sqrt{Var_{\mathcal{K}}(V)}&\leq\sqrt{\phi(s,a)Var_{\mathcal{K}}(V)}\\
	&=\sqrt{\sum_{k\in\mathcal{K}}\phi_k(s,a)Var_{s_k,a_k}(V)}\\
	&=\sqrt{\sum_{k\in\mathcal{K}}^d\phi_k(s,a)(P(s_k,a_k)V^2-(P(s_k,a_k)V)^2)}\\
	&=\sqrt{\sum_{k\in\mathcal{K}}^d\phi_k(s,a)P(s_k,a_k)V^2-\sum_{k\in\mathcal{K}}^d\phi_k(s,a)(P(s_k,a_k)V)^2}\\
	&\leq \sqrt{P(s,a)V^2-\big(\sum_{k\in\mathcal{K}}^d\phi_k(s,a)P(s_k,a_k)V\big)^2}\\
	&=\sqrt{P(s,a)V^2-(P(s,a)V)^2}\\
	&=\sqrt{Var_{s,a}(V)}.
	\end{align*}
	The two inequalities are due to Jensen's inequality and other steps are from $P(s,a)=\sum_{k\in\mathcal{K}}\phi_k(s,a)P(s_k,a_k)$, which is a row vector version of $P=\Phi P_\mathcal{K}$.
\end{proof}

\begin{lemma}
	For any policy $\pi$ and $V^\pi$ is the value function in a MDP with transition $P$,
	$$\big\|(I-\gamma P^\pi)^{-1}\sqrt{Var_P(V^\pi)}\big\|_\infty\leq\sqrt{\frac{2}{(1-\gamma)^3}}.$$
	\label{l4}
\end{lemma}

\begin{proof}
	Since $(1-\gamma)(I-\gamma P^\pi)^{-1}$ is a probability transition matrix, we can apply Jensen's inequality,
	\begin{align*}
	\|(I-\gamma P^\pi)^{-1}\sqrt{Var_P(V^\pi)}\|_\infty&\leq\sqrt{\frac{1}{1-\gamma}}\sqrt{\|(I-\gamma P^\pi)^{-1}Var_P(V^\pi)\|_\infty}\\
	&\leq\sqrt{\frac{2}{1-\gamma}}\sqrt{\|(I-\gamma^2 P^\pi)^{-1}Var_P(V^\pi)\|_\infty}\\
	&\leq \sqrt{\frac{2}{1-\gamma}}\sqrt{\Sigma^\pi\cdot\frac{1}{\gamma^2}}\\
	&\leq \sqrt{\frac{2}{(1-\gamma)^3}}
	\end{align*}
	The definition of $\Sigma$ and a detailed proof is given in \cite{Yang2019}.
\end{proof}

\begin{lemma}
	For any value function $V_1$ and $V_2$, we have
	$$\sqrt{Var_{s,a}(V_1+V_2)}\leq\sqrt{Var_{s,a}(V_1)}+\sqrt{Var_{s,a}(V_2)}.$$
	\label{l5}
\end{lemma}

\begin{proof}
    This Lemma is the triangle inequality for variance.
\end{proof}

\begin{lemma}
	Let $u^\pi=\gamma(\widehat{P}(s,a)-P(s,a))\widehat{V}^\pi$ and $u^*=\gamma(\widehat{P}(s,a)-P(s,a))\widehat{V}^*$, then we have
	$$\widehat{Q}^\pi=\widetilde{Q}^\pi_{u^\pi},\ \widehat{Q}^*=\widetilde{Q}^{\widehat{\pi}^*}_{u^*}=\widetilde{Q}^*_{u^*},\ -\frac{1}{1-\gamma}\leq u^\pi\leq\frac{1}{1-\gamma},\ -\frac{1}{1-\gamma}\leq u^*\leq\frac{1}{1-\gamma}.$$
	\label{l6}
\end{lemma}

\begin{proof}
Using the Bellman equation $Q^\pi=(I-\gamma P^\pi)^{-1}r$, we have
	\begin{align*}
	\begin{split}
	\widetilde{Q}^\pi_{u^\pi}&=(I-\gamma\widetilde{P}^\pi)^{-1}(r+\Phi^{s,a}\cdot\gamma(\widehat{P}(s,a)-P(s,a))\widehat{V}^\pi)\\
	&=(I-\gamma\widetilde{P}^\pi)^{-1}((I-\gamma\widehat{P}^\pi)\widehat{Q}^\pi+\gamma\Phi(\widehat{P}_\mathcal{K}-\widetilde{P}_\mathcal{K})\widehat{V}^\pi)\\
	&=(I-\gamma\widetilde{P}^\pi)^{-1}((I-\gamma\widehat{P}^\pi)\widehat{Q}^\pi+\gamma(\widehat{P}-\widetilde{P})\widehat{V}^\pi)\\
	&=(I-\gamma\widetilde{P}^\pi)^{-1}(I-\gamma\widetilde{P}^\pi)\widehat{Q}^\pi\\
	&=\widehat{Q}^\pi.
	\end{split}
	\end{align*}
	Similarly, we have $\widehat{Q}^*=\widetilde{Q}^{\widehat{\pi}}_{u^*}$. By the definition of $\widehat{Q}^*$ and the sufficient condition for optimal value, we have
	$$\widetilde{Q}^{\widehat{\pi}^*}_{u^*}(s,\widehat{\pi}^*(s))=\widehat{Q}^*(s,\widehat{\pi}(s))=\max_a\widehat{Q}^*(s,a)=\max_a\widetilde{Q}^{\widehat{\pi}^*}_{u^*}(s,a).$$
	So, $\widehat{\pi}^*$ is the optimal policy in $\widetilde{\mathcal{M}}_{s,a,u^*}$ and $\widehat{Q}^*=\widetilde{Q}^{\widehat{\pi}}_{u^*}=\widetilde{Q}^*_{u^*}$.
	
	As $0\leq\widehat{V}^\pi\leq\frac{1}{1-\gamma}$ and  $0\leq\widehat{V}^*\leq\frac{1}{1-\gamma}$, we can immediately derive that $-\frac{1}{1-\gamma}\leq u^\pi\leq\frac{1}{1-\gamma},\ -\frac{1}{1-\gamma}\leq u^*\leq\frac{1}{1-\gamma}$.
\end{proof}

\begin{remark}
    Lemma \ref{l6} shows that we can tune a bounded scalar $u$ in $\widetilde{M}_{s,a,u}$ to recover $\widehat{Q}$, which implies $\widehat{Q}$, as a function of $\widehat{P}(s,a)$, lies in a one-dimensional manifold in $\mathbb{R}^{|\mathcal{S}||\mathcal{A}|}$.
\end{remark}

\begin{lemma}
	For all $u_1,u_2\in\mathbb{R}$ and policy $\pi$,
	$$\left\|\widetilde{Q}^\pi_{u_1}-\widetilde{Q}^\pi_{u_2}\right\|_\infty\leq\frac{1}{1-\gamma}\left|u_1-u_2\right|,\left\|\widetilde{Q}^*_{u_1}-\widetilde{Q}^*_{u_2}\right\|_\infty\leq\frac{1}{1-\gamma}\left|u_1-u_2\right|.$$
	\label{l7}
\end{lemma}

\begin{proof}
Using Bellman equation, we have
	\begin{align*}
	\left\|\widetilde{Q}^\pi_{u_1}-\widetilde{Q}^\pi_{u_2}\right\|_\infty&=\left\|\left(I-\gamma\widetilde{P}^\pi\right)^{-1}\left(r+u_1\Phi^{s,a}\right)-\left(I-\gamma\widetilde{P}^\pi\right)^{-1}\left(r+u_2\Phi^{s,a}\right)\right\|_\infty\\
	&=\left\|(u_1-u_2)\left(I-\gamma\widetilde{P}^\pi\right)^{-1}\Phi^{s,a}\right\|_\infty\\
	&\leq\left|u_1-u_2\right|\frac{1}{1-\gamma}\left\|\Phi^{s,a}\right\|_\infty\\
	&\leq\left|u_1-u_2\right|\frac{1}{1-\gamma}.
	\end{align*}
	Now we prove the second claim. Set $\pi=\pi_{u_1}^*$ and $\pi=\pi_{u_2}^*$, we have
	$$\left\|\widetilde{Q}^*_{u_1}-\widetilde{Q}^{\pi_{u_1}^*}_{u_2}\right\|_\infty\leq\left|u_1-u_2\right|\frac{1}{1-\gamma},\left\|\widetilde{Q}^{\pi_{u_2}^*}_{u_1}-\widetilde{Q}^*_{u_2}\right\|_\infty\leq\left|u_1-u_2\right|\frac{1}{1-\gamma}.$$
	Using the property of optimal value, we have
	$$-\left|u_1-u_2\right|\frac{1}{1-\gamma}\mathbf{1}\leq\widetilde{Q}^{\pi_{u_2}^*}_{u_1}-\widetilde{Q}^*_{u_2}\leq\widetilde{Q}^*_{u_1}-\widetilde{Q}^*_{u_2}\leq\widetilde{Q}^*_{u_1}-\widetilde{Q}^{\pi_{u_1}^*}_{u_2}\leq\big|u_1-u_2\big|\frac{1}{1-\gamma}\mathbf{1},$$
	which implies $\left\|\widetilde{Q}^*_{u_1}-\widetilde{Q}^*_{u_2}\right\|_\infty\leq\left|u_1-u_2\right|\frac{1}{1-\gamma}.$
\end{proof}

\begin{remark}
    Lemma \ref{l7} shows that $\widetilde{Q}_u$ is robust to $u$. This property implies that an $\epsilon$-net on $u$ can form an $\epsilon/(1-\gamma)$-net on $\widehat{Q}_u$.
\end{remark}

\begin{lemma}
	For a given finite set $B_{s,a}^\pi\subset U_{s,a}^\pi\cap[-\frac{1}{1-\gamma},\frac{1}{1-\gamma}]$ and $\delta\geq0$, with probability greater than $1-\delta$, it holds for all $u\in B_{s,a}^\pi$ that
	$$\left|\left(P(s,a)-\widehat{P}(s,a)\right)\widetilde{V}^\pi_u\right|\leq\sqrt{\frac{2\log\left(4\left|B_{s,a}^\pi\right|/\delta\right)}{N}}\sqrt{Var_{s,a}(\widetilde{V}_u^\pi)}+\frac{2\log\left(4\left|B_{s,a}^\pi\right|/\delta\right)}{(1-\gamma)3N}.$$
	Similarly, For a given finite set $B_{s,a}^*\subset U_{s,a}^*\cap[-\frac{1}{1-\gamma},\frac{1}{1-\gamma}]$ and $\delta\geq0$, with probability greater than $1-\delta$, it holds for all $u\in B_{s,a}^*$ that
	$$\left|\left(P(s,a)-\widehat{P}(s,a)\right)\widetilde{V}^*_u\right|\leq\sqrt{\frac{2\log\left(4\left|B_{s,a}^*\right|/\delta\right)}{N}}\sqrt{Var_{s,a}(\widetilde{V}_u^*)}+\frac{2\log\left(4\left|B_{s,a}^*\right|/\delta\right)}{(1-\gamma)3N}.$$
	\label{l8}
\end{lemma}

\begin{proof}
	This is the direct application of Beinstein's inequality as $\widetilde{V}^\pi_u$ and $\widetilde{V}^*_u$ is independent of $\widehat{P}(s,a)$.
\end{proof}

\begin{lemma}
	For a given finite set $B_{s,a}^{\pi^*}\subset U_{s,a}^{\pi^*}\cap[-\frac{1}{1-\gamma},\frac{1}{1-\gamma}]$ and $B_{s,a}^*\subset U_{s,a}^*\cap[-\frac{1}{1-\gamma},\frac{1}{1-\gamma}]$ and $\delta\geq0$, with probability greater than $1-2\delta$, it holds for all $u\in B_{s,a}^{\pi^*}$ that
	\begin{multline*}
	\left|\left(P(s,a)-\widehat{P}(s,a)\right)\widehat{V}^{\pi^*}\right|\leq\sqrt{\frac{2\log\left(4\left|B_{s,a}^{\pi^*}\right|/\delta\right)}{N}}\sqrt{Var_{s,a}(\widehat{V}^{\pi^*})}+\frac{2\log\left(4\left|B_{s,a}^{\pi^*}\right|/\delta\right)}{(1-\gamma)3N}\\+\min_{u\in B_{s,a}^{\pi^*}}\left|u^{\pi^*}-u\right|\frac{1}{1-\gamma}\left(1+\sqrt{\frac{2\log\left(4\left|B_{s,a}^{\pi^*}\right|/\delta\right)}{N}}\right),
	\end{multline*}
	\begin{multline*}
	\left|\left(P(s,a)-\widehat{P}(s,a)\right)\widehat{V}^*\right|\leq\sqrt{\frac{2\log\left(4\left|B_{s,a}^{*}\right|/\delta\right)}{N}}\sqrt{Var_{s,a}(\widehat{V}^*)}+\frac{2\log\left(4\left|B_{s,a}^{*}\right|/\delta\right)}{(1-\gamma)3N}\\
	+\min_{u\in B_{s,a}^*}\left|u^*-u\right|\frac{1}{1-\gamma}\left(1+\sqrt{\frac{2\log\left(4\left|B_{s,a}^*\right|/\delta\right)}{N}}\right).
	\end{multline*}
	\label{l9}
\end{lemma}

\begin{proof}
	For the first claim, we have
	\begin{align*}
	&\left|(P(s,a)-\widehat{P}(s,a))\widehat{V}^{\pi^*}\right|
	\\\overset{(a)}{\leq}&\left|(P(s,a)-\widehat{P}(s,a))\widetilde{V}^{\pi^*}_u\right|+\left|(P(s,a)-\widehat{P}(s,a))(\widehat{V}^{\pi^*}-\widetilde{V}^{\pi^*}_u)\right|
	\\\overset{\mathrm{(b)}}{\leq}& \sqrt{\frac{2\log(4\left|B_{s,a}^{\pi^*}\right|/\delta)}{N}}\sqrt{Var_{s,a}(\widetilde{V}_u^{\pi^*})}+\frac{2\log(4\left|B_{s,a}^{\pi^*}\right|/\delta)}{(1-\gamma)3N}+\left\|\widehat{V}^{\pi^*}-\tilde{V}^{\pi^*}_u\right\|_\infty
	\\\overset{\mathrm{(c)}}{\leq}& \sqrt{\frac{2\log(4\left|B_{s,a}^{\pi^*}\right|/\delta)}{N}}\bigg(\sqrt{Var_{s,a}(\widehat{V}^{\pi^*})}
	+\sqrt{Var_{s,a}(\widetilde{V}_u^{\pi^*}-\widehat{V}^{\pi^*})}\bigg)\\& +\frac{2\log(4\left|B_{s,a}^{\pi^*}\right|/\delta)}{(1-\gamma)3N}+\left\|\widehat{V}^{\pi^*}-\widetilde{V}^{\pi^*}_u\right\|_\infty\\
	\overset{\mathrm{(d)}}{\leq}& \sqrt{\frac{2\log(4\left|B_{s,a}^{\pi^*}\right|/\delta)}{N}}\sqrt{Var_{s,a}(\widehat{V}^{\pi^*})} +\frac{2\log(4\left|B_{s,a}^{\pi^*}\right|/\delta)}{(1-\gamma)3N} \\&+\left\|\widehat{V}^{\pi^*}-\widetilde{V}^{\pi^*}_u\right\|_\infty\left(1+\sqrt{\frac{2\log(4\left|B_{s,a}^\pi\right|/\delta)}{N}}\right)\\
	\overset{\mathrm{(e)}}{\leq}& \sqrt{\frac{2\log(4\left|B_{s,a}^{\pi^*}\right|/\delta)}{N}}\sqrt{Var_{s,a}(\widehat{V}^{\pi^*})} +\frac{2\log(4\left|B_{s,a}^{\pi^*}\right|/\delta)}{(1-\gamma)3N}\\&\ \ \ +\left|u^{\pi^*}-u\right|\frac{1}{1-\gamma}\left(1+\sqrt{\frac{2\log(4\left|B_{s,a}^{\pi^*}\right|/\delta)}{N}}\right).
	\end{align*}
	(a) is due to triangle inequality, (b) is from Lemma \ref{l8}, (c) is from Lemma \ref{l5}, (d) is due to the fact that $\sqrt{Var(V)}\leq V$ and (e) is from Lemma \ref{l7}. As this equality holds for all $u\in B_{s,a}^{\pi^*}$, we can take minimum in the RHS, which proves the first claim. The second claims can proved in the same manner.
\end{proof}

\begin{lemma}
	For any given $\epsilon$ and all $(s,a)\in \mathcal{K}$, with probability larger than $1-\delta$,
	\begin{multline*}
	\left|\left(P\left(s,a\right)-\widehat{P}(s,a)\right)\widehat{V}^{\pi^*}\right|\leq\sqrt{\frac{2\log\left(32K/\delta\left(1-\gamma\right)^3\epsilon\right)}{N}}\sqrt{Var_{s,a}(\widehat{V}^{\pi^*})}\\+\frac{2\log\left(32K/\delta\left(1-\gamma\right)^3\epsilon\right)}{\left(1-\gamma\right)3N}
	+\left(\sqrt{\frac{2\log\left(32K/\delta\left(1-\gamma\right)^3\epsilon\right)}{N}}+1\right)\frac{\epsilon\left(1-\gamma\right)}{4},
	\end{multline*}
	\begin{multline*}
	\left|\left(P\left(s,a\right)-\widehat{P}\left(s,a\right)\right)\widehat{V}^{*}\right|\leq\sqrt{\frac{2\log\left(32K/\delta\left(1-\gamma\right)^3\epsilon\right)}{N}}\sqrt{Var_{s,a}\left(\widehat{V}^*\right)}\\+\frac{2\log\left(32K/\delta\left(1-\gamma\right)^3\epsilon\right)}{\left(1-\gamma\right)3N}
	+\left(\sqrt{\frac{2\log\left(32K/\delta\left(1-\gamma\right)^3\epsilon\right)}{N}}+1\right)\frac{\epsilon\left(1-\gamma\right)}{4}.
	\end{multline*}
	For simplicity, we set $c\left(\delta,\gamma,\epsilon\right)=2\log\left(32K/\delta\left(1-\gamma\right)^3\epsilon\right)$ and we use $c$ to represent $c\left(\delta,\gamma,\epsilon\right)$ as it includes only log factors.
	\label{l10}
\end{lemma}

\begin{proof}
	We set $B_{s,a}^{\pi^*}$ to be the evenly spaced elements in the interval $U^*_{s,a}\cap[-\frac{1}{1-\gamma},\frac{1}{1-\gamma}]$ and $\left|B_{s,a}^{\pi^*}\right|=\frac{4}{(1-\gamma)^3\epsilon}$. Then for any $u'\in U^*_{s,a}\cap[-\frac{1}{1-\gamma},\frac{1}{1-\gamma}]$, we have $\min_{u\in B_{s,a}^{\pi^*}}\left|u'-u\right|\leq(1-\gamma)^2\epsilon/4$. Note that $u^{\pi^*}\in U^*_{s,a}\cap[-\frac{1}{1-\gamma},\frac{1}{1-\gamma}]$. Then Lemma \ref{l9} implies this result. Similarly we can prove the second claim.
\end{proof}

\begin{lemma}
	With probability larger than $1-\delta$,
	$$
	\left|(P-\widehat{P})\widehat{V}^{\pi^*}\right|\leq \sqrt{\frac{c}{N}}\sqrt{Var_{P}(\widehat{V}^{\pi^*})}+\left[\frac{c}{\left(1-\gamma\right)3N}
	+\left(\sqrt{\frac{c}{N}}+1\right)\frac{\epsilon\left(1-\gamma\right)}{4}\right]\mathbf{1},
	$$
	$$
	\left|(P-\widehat{P})\widehat{V}^{*}\right|\leq\sqrt{\frac{c}{N}}\sqrt{Var_{P}(\widehat{V}^*)}+\left[\frac{c}{\left(1-\gamma\right)3N}
	+\left(\sqrt{\frac{c}{N}}+1\right)\frac{\epsilon(1-\gamma)}{4}\right]\mathbf{1}.
	$$
	\label{l11}
\end{lemma}

\begin{proof}
	\begin{align*}
	\left|(P-\widehat{P})\widehat{V}^{\pi^*}\right|&=\left|\Phi(P_{\mathcal{K}}-\widehat{P}_{\mathcal{K}})\widehat{V}^*\right|\\
	&\overset{\mathrm{(a)}}{\leq}\Phi\left|(P_{\mathcal{K}}-\widehat{P}_{\mathcal{K}})\widehat{V}^*\right|\\
	&\overset{\mathrm{(b)}}{\leq}\Phi\sqrt{\frac{c}{N}}\sqrt{Var_{\mathcal{K}}(\widehat{V}^{\pi^*})}+\Phi\left[\frac{c}{\left(1-\gamma\right)3N}
	+\left(\sqrt{\frac{c}{N}}+1\right)\frac{\epsilon\left(1-\gamma\right)}{4}\right] \mathbf{1}\\
	&\overset{\mathrm{(c)}}{\leq} \sqrt{\frac{c}{N}}\sqrt{Var_{P}\left(\widehat{V}^{\pi^*}\right)}+\left[\frac{c}{\left(1-\gamma\right)3N}
	+\left(\sqrt{\frac{c}{N}}+1\right)\frac{\epsilon\left(1-\gamma\right)}{4}\right]\mathbf{1}.
	\end{align*}
	(a) is due to $\Phi$ is non-negative, (b) is from Lemma \ref{l10} and (c) is from Lemma \ref{3}. The second claim can be proved in the same manner.
\end{proof}

\begin{lemma}
	With probability larger than $1-\delta$, and $\widehat{\pi}$ is a $\epsilon_\mathrm{PS}$-optimal policy in $\widehat{M}$.
	$$
	\left\|Q^*-\widehat{Q}^{\pi^*}\right\|_\infty\leq\frac{\gamma}{1-\alpha}\left(\sqrt{\frac{c}{N\left(1-\gamma\right)^3}}+\frac{c}{\left(1-\gamma\right)^23N}
	+\left(\sqrt{\frac{c}{N}}+1\right)\frac{\epsilon}{4}\right),
	$$
	$$
	\left\|Q^{\widehat{\pi}}-\widehat{Q}^{\widehat{\pi}}\right\|_\infty\leq\frac{\gamma}{1-\alpha}\left(\sqrt{\frac{c}{N\left(1-\gamma\right)^3}}+\frac{c}{\left(1-\gamma\right)^23N}
	+\left(\sqrt{\frac{c}{N}}+1\right)\left(\frac{\epsilon}{4}+\frac{\epsilon_\mathrm{PS}}{1-\gamma}\right)\right).
	$$
	where $\alpha=\alpha(\delta,\gamma,\epsilon,N)=\sqrt{\frac{c(\delta,\gamma,\epsilon)}{N(1-\gamma)^2}}.$
	\label{l12}
\end{lemma}

\begin{proof}
	\begin{align*}
	&\left\|Q^*-\widehat{Q}^{\pi^*}\right\|_\infty
	\\=&\left\|\left(I-\gamma P^{\pi^*}\right)^{-1}\left(\widehat{P}-P\right)\widehat{V}^{\pi^*}\right\|_\infty\\
	\overset{\mathrm{(a)}}{\leq}& \bigg\|\left(I-\gamma P^{\pi^*}\right)^{-1}\bigg[\sqrt{\frac{c}{N}}\sqrt{Var_{P}\left(\widehat{V}^{\pi^*}\right)}+\frac{c}{\left(1-\gamma\right)3N}
	+\left(\sqrt{\frac{c}{N}}+1\right)\frac{\epsilon}{4\left(1-\gamma\right)}\bigg]\bigg\|_\infty\\
	\overset{\mathrm{(b)}}{\leq}& \sqrt{\frac{c}{N\left(1-\gamma\right)^3}}+\sqrt{\frac{c}{N}}\frac{\left\|Q^*-\widehat{Q}^{\pi^*}\right\|}{1-\gamma}+\frac{c}{\left(1-\gamma\right)^23N}
	+\left(\sqrt{\frac{c}{N}}+1\right)\frac{\epsilon}{4}.
	\end{align*}
	(a) is from Lemma \ref{l11} and (b) is from Lemma \ref{l4}. Solving for $\left\|Q^*-\widehat{Q}^{\pi^*}\right\|_\infty$ proves the first claim. 
	\begin{align*}
	&\left\|Q^{\widehat{\pi}}-\widehat{Q}^{\widehat{\pi}}\right\|_\infty
	\\=&\left\|\left(I-\gamma P^{\widehat{\pi}}\right)^{-1}\left(\widehat{P}-P\right)\widehat{V}^{\widehat{\pi}}\right\|_\infty\\
	\overset{\mathrm{(a)}}{\leq}& \left\|\left(I-\gamma P^{\widehat{\pi}}\right)^{-1}\left(\widehat{P}-P\right)\widehat{V}^{*}\right\|_\infty+\left\|\left(I-\gamma P^{\widehat{\pi}}\right)^{-1}\left(\widehat{P}-P\right)\left(\widehat{V}^{\widehat{\pi}}-\widehat{V}^{*}\right)\right\|_\infty\\
	\overset{\mathrm{(b)}}{\leq}&\bigg\|\left(I-\gamma P^{\widehat{\pi}}\right)^{-1}\bigg[\sqrt{\frac{c}{N}}\sqrt{Var_{P}(\widehat{V}^{*})}+\frac{c}{(1-\gamma)3N}+\left(\sqrt{\frac{c}{N}}+1\right)\frac{\epsilon}{4(1-\gamma)}\bigg]\bigg\|_\infty\\&+\frac{\epsilon_\mathrm{PS}}{1-\gamma}\\
	\overset{\mathrm{(c)}}{\leq}&\bigg\|\left(I-\gamma P^{\widehat{\pi}}\right)^{-1}\bigg[\sqrt{\frac{c}{N}}\sqrt{Var_{P}\left(V^{\widehat{\pi}}+\widehat{V}^{\widehat{\pi}}-V^{\widehat{\pi}}+\widehat{V}^*-\widehat{V}^{\widehat{\pi}}\right)}+\frac{c}{(1-\gamma)3N}
	\\&+\left(\sqrt{\frac{c}{N}}+1\right)\frac{\epsilon}{4(1-\gamma)}\bigg]\bigg\|_\infty+\frac{\epsilon_\mathrm{PS}}{1-\gamma}\\
	\overset{\mathrm{(d)}}{\leq}&\left\|\sqrt{\frac{c}{N}}\left(I-\gamma P^{\widehat{\pi}}\right)^{-1}\sqrt{Var_{P}\left(\widehat{V}^{\widehat{\pi}}\right)}\right\|_\infty+\sqrt{\frac{c}{N}}\frac{\left\|Q^{\widehat{\pi}}-\widehat{Q}^{\widehat{\pi}}\right\|_\infty}{1-\gamma}+\frac{c}{(1-\gamma)^23N}
	\\&+\left(\sqrt{\frac{c}{N}}+1\right)\left(\frac{\epsilon}{4}+\frac{\epsilon_\mathrm{PS}}{1-\gamma}\right)\\
	\overset{\mathrm{(e)}}{\leq}& \sqrt{\frac{c}{N(1-\gamma)^3}}+\sqrt{\frac{c}{N}}\frac{\left\|Q^{\widehat{\pi}}-\widehat{Q}^{\widehat{\pi}}\right\|_\infty}{1-\gamma}+\frac{c}{(1-\gamma)^23N}
	+\left(\sqrt{\frac{c}{N}}+1\right)\left(\frac{\epsilon}{4}+\frac{\epsilon_\mathrm{PS}}{1-\gamma}\right).
	\end{align*}
	(a), (c), (d) are due to triangle inequality, (b) is from Lemma \ref{11} and (e) is from Lemma \ref{l4}. Solving for $\left\|Q^{\widehat{\pi}}-\widehat{Q}^{\widehat{\pi}}\right\|_\infty$ proves the second claim. 
\end{proof}

\begin{proof}[Proof of Theorem 1]
    From Lemma \ref{l1}, with probability larger than $1-\delta$, we have
    \begin{align*}
        \left\|Q^*-Q^{\widehat{\pi}}\right\|_\infty&\overset{\mathrm{(a)}}{\leq}\left\|Q^*-\widehat{Q}^{\pi^*}\right\|_\infty+\left\|\widehat{Q}^{\widehat{\pi}}-Q^{\widehat{\pi}}\right\|_\infty+\epsilon_\mathrm{PS}\\
        &\overset{\mathrm{(b)}}{\leq}\frac{\gamma}{1-\alpha}\bigg[2\left(\sqrt{\frac{c}{N(1-\gamma)^3}}+\frac{c}{(1-\gamma)^23N}
	+\left(\sqrt{\frac{c}{N}}+1\right)\frac{\epsilon}{4}\right)
	\\&+\left(\sqrt{\frac{c}{N}}+1\right)\frac{\epsilon_\mathrm{PS}}{1-\gamma}\bigg]+\epsilon_\mathrm{PS}\\
    \end{align*}
    (a) is from Lemma \ref{l1} and (b) is from Lemma \ref{l12}. For $N\geq\frac{C\log(CK(1-\gamma)^{-1}\delta^{-1}\epsilon^{-1})}{(1-\gamma)^3\epsilon^2}$ with proper constant $C$, we have $\frac{\gamma}{1-\alpha(\delta,\gamma,\epsilon,N)}\left(\sqrt{\frac{c(\delta,\gamma,\epsilon)}{N}}+1\right)\leq2$, thus
    \[\left\|V^*-V^{\widehat{\pi}}\right\|_\infty\leq\epsilon+\frac{3\epsilon_\mathrm{PS}}{1-\gamma},\]
    which completes the proof.

\end{proof}

Now we prove Theorem 2, where the transition model $P$ can be approximated by linear transition model, i.e. $P=\Bar{P}+\Xi=\Phi\Bar{P}_\mathcal{K}+\Xi$, where $\Bar{P}$ is a linear transition model and $\Xi$ is the approximation error matrix. We set $\xi=\left\|\Xi\right\|_{1,\infty}.$

\begin{lemma}
	For any value function $V$ and state action pair $(s,a)$,
	$$\phi(s,a)\sqrt{Var_{\mathcal{K}}(V)}=\sum_{k\in\mathcal{K}}\phi_k(s,a)\sqrt{Var_{s_k,a_k}(V)}\leq\sqrt{Var_{s,a}(V)}+2\sqrt{\frac{3\xi}{(1-\gamma)^2}}.$$
	\label{l13}
\end{lemma}

\begin{proof}
	Since $\phi(s,a)$ is a probability transition matrix, we can use Jensen's inequality here.
	\begin{align*}
	&\phi(s,a)\sqrt{Var_{\mathcal{K}}(V)}
	\\\leq&\sqrt{\phi(s,a)Var_{\mathcal{K}}(V)}\\
	=&\sqrt{\sum_{k\in\mathcal{K}}\phi_k(s,a)Var_{s_k,a_k}(V)}\\
	=&\sqrt{\sum_{k\in\mathcal{K}}\phi_k(s,a)\left(P(s_k,a_k)V^2-(P(s_k,a_k)V\right)^2)}\\
	=&\sqrt{\sum_{k\in\mathcal{K}}\phi_k(s,a)P(s_k,a_k)V^2-\sum_{k\in\mathcal{K}}\phi_k(s,a)\left(P(s_k,a_k)V\right)^2}\\
	=&\sqrt{\sum_{k\in\mathcal{K}}\phi_k(s,a)\left(\Bar{P}(s_k,a_k)V^2+\Xi(s_k,a_k)V^2\right)-\sum_{k\in\mathcal{K}}\phi_k(s,a)\left(\left(\Bar{P}(s_k,a_k)+\Xi(s_k,a_k)\right)V\right)^2}\\
	\leq&\sqrt{\sum_{k\in\mathcal{K}}\phi_k(s,a)\Bar{P}(s_k,a_k)V^2-\sum_{k\in\mathcal{K}}\phi_k(s,a)\left(\Bar{P}(s_k,a_k)V\right)^2}+\sqrt{\frac{3\xi}{(1-\gamma)^2}}\\
	\leq& \sqrt{\Bar{P}(s,a)V^2-\sum_{k\in\mathcal{K}}\phi_k(s,a)\left(\Bar{P}(s_k,a_k)V\right)^2}+\sqrt{\frac{3\xi}{(1-\gamma)^2}}\\
=&\sqrt{\Bar{P}(s,a)V^2-\left(\Bar{P}(s,a)V\right)^2}+\sqrt{\frac{3\xi}{(1-\gamma)^2}}\\
	\leq&\sqrt{P(s,a)V^2-\left(P(s,a)V\right)^2}+2\sqrt{\frac{3\xi}{(1-\gamma)^2}}\\
	=&\sqrt{Var_{s,a}(V)}+2\sqrt{\frac{3\xi}{(1-\gamma)^2}}.
	\end{align*}
\end{proof}

\begin{lemma}
	With probability larger than $1-\delta$, and $\widehat{\pi}$ is a $\epsilon_\mathrm{PS}$-optimal policy in $\widehat{\mathcal{K}}$.
	\begin{multline*}
	\left\|Q^*-\widehat{Q}^{\pi^*}\right\|_\infty\leq\frac{\gamma}{1-\alpha}\left(\sqrt{\frac{c}{N(1-\gamma)^3}}+\frac{c}{(1-\gamma)^23N}
	+\left(\sqrt{\frac{c}{N}}+1\right)\frac{\epsilon}{4}+8\sqrt{\frac{\xi}{(1-\gamma)^4}}\right),
	\end{multline*}
	\begin{multline*}
	\left\|Q^{\widehat{\pi}}-\widehat{Q}^{\hat{\pi}}\right\|_\infty\leq\frac{\gamma}{1-\alpha}\bigg(\sqrt{\frac{c}{N(1-\gamma)^3}}+\frac{c}{(1-\gamma)^23N}
	+\left(\sqrt{\frac{c}{N}}+1\right)\left(\frac{\epsilon}{4}+\frac{\epsilon_\mathrm{PS}}{1-\gamma}\right)
	\\+8\sqrt{\frac{\xi}{(1-\gamma)^4}}\bigg).
	\end{multline*}
	\label{l14}
\end{lemma}

\begin{proof}
	\begin{align*}
	&\left\|Q^*-\widehat{Q}^{\pi^*}\right\|_\infty\\=&\left\|\left(I-\gamma P^{\pi^*}\right)^{-1}(\widehat{P}-P)\widehat{V}^{\pi^*}\right\|_\infty\\
	=&\left\|\left(I-\gamma P^{\pi^*}\right)^{-1}\left(\Phi\widehat{P}_\mathcal{K}-\Phi\Bar{P}_\mathcal{K}-\Xi_\mathcal{K}\right)\widehat{V}^{\pi^*}\right\|_\infty\\
	\leq&\left\|\left(I-\gamma P^{\pi^*}\right)^{-1}\Phi\left(\widehat{P}_\mathcal{K}-\Bar{P}_\mathcal{K}-\Xi_\mathcal{K}\right)\widehat{V}^{\pi^*}\right\|_\infty+\frac{2\xi}{(1-\gamma)^2}\\
	=&\left\|\left(I-\gamma P^{\pi^*}\right)^{-1}\Phi\left(\widehat{P}_\mathcal{K}-P_\mathcal{K}\right)\widehat{V}^{\pi^*}\right\|_\infty+\frac{2\xi}{(1-\gamma)^2}\\
	\leq& \bigg\|\left(I-\gamma P^{\pi^*}\right)^{-1}\bigg[\sqrt{\frac{c}{N}}\left(\sqrt{Var_{P}\left(\widehat{V}^{\pi^*}\right)}+2\sqrt{\frac{3\xi}{(1-\gamma)^2}}\right)+\frac{c}{(1-\gamma)3N}
	\\&+\left(\sqrt{\frac{c}{N}}+1\right)\frac{\epsilon}{4(1-\gamma)}\bigg]\bigg\|_\infty+\frac{2\xi}{(1-\gamma)^2}\\
	\leq& \sqrt{\frac{c}{N(1-\gamma)^3}}+\sqrt{\frac{c}{N}}\frac{\left\|Q^*-\widehat{Q}^{\pi^*}\right\|}{1-\gamma}+\frac{c}{(1-\gamma)^23N}
	+\left(\sqrt{\frac{c}{N}}+1\right)\frac{\epsilon}{4}+8\sqrt{\frac{\xi}{(1-\gamma)^4}}.
	\end{align*}
	Solving for $\left\|Q^*-\widehat{Q}^{\pi^*}\right\|_\infty$ proves the first claim. The second claim can be proved in a similar manner.
\end{proof}

\begin{proof}[Proof of Theorem 2]
    From Lemma \ref{l1}, with probability larger than $1-\delta$, we have
    \begin{align*}
        &\left\|Q^*-Q^{\widehat{\pi}}\right\|_\infty\\
        \leq&\left\|Q^*-\widehat{Q}^{\pi^*}\right\|_\infty+\left\|\widehat{Q}^{\widehat{\pi}}-Q^{\widehat{\pi}}\right\|_\infty+\epsilon_\mathrm{PS}\\
        \leq&\frac{\gamma}{1-\alpha}\bigg[2\left(\sqrt{\frac{c}{N(1-\gamma)^3}}+\frac{c}{(1-\gamma)^23N}
	+\left(\sqrt{\frac{c}{N}}+1\right)\frac{\epsilon}{4}+8\sqrt{\frac{\xi}{(1-\gamma)^4}}\right)\\
	&+\left(\sqrt{\frac{c}{N}}+1\right)\frac{\epsilon_\mathrm{PS}}{1-\gamma}\bigg]+\epsilon_\mathrm{PS}\\
    \end{align*}
    For $N\geq\frac{C\log(CK(1-\gamma)^{-1}\delta^{-1})\epsilon^{-1}}{(1-\gamma)^3\epsilon^2}$ with proper constant $C$, we have $\frac{\gamma}{1-\alpha}\left(\sqrt{\frac{c}{N}}+1\right)\leq2$, thus
    \[\left\|V^*-V^{\widehat{\pi}}\right\|_\infty\leq\epsilon+\frac{3\epsilon_\mathrm{PS}}{1-\gamma}+\frac{16\sqrt{\xi}}{(1-\gamma)^2,}\]
which completes the proof.
\end{proof}

\subsection{Sample Complexity for Finite Horizon MDP}

Here we prove the sample complexity result for FHMDP using the auxiliary MDP technique. The difference with DMDP is that here we need to tune the reward in each time step.

\begin{definition}
	(Auxiliary Model) For an estimated transition model $\widehat{\mathcal{M}}=(\mathcal{S},\mathcal{A},\widehat{P}=\Phi\widehat{P}_\mathcal{K},r,\gamma)$ and a given anchor state pair $(s,a)$, the auxiliary transition model is $\widetilde{\mathcal{M}}_{s,a,u}=(\mathcal{S},\mathcal{A},\widetilde{P}=\Phi\widetilde{P}_K,\widetilde{r}_{h}^u=r+u_h\Phi^{s,a},\gamma)$, where
	$$\widetilde{P}_\mathcal{K}(s',a')=
	\begin{cases}
	\widehat{P}(s',a')& \mathrm{if}\ (s',a')\neq(s,a)\\
	P(s,a)& \mathrm{otherwise,}
	\end{cases}$$
	$\Phi^{s,a}$ is the column of $\Phi$ that corresponds to anchor state $(s,a)$, $\widetilde{r}_{h}^u$ is the reward in step $h$ and $u=(u_0,u_1,\cdots,u_{H-1})$ is a $H$ dimensional vector that will be determined latter.
\end{definition}

\begin{remark}
	The reward in $\widetilde{\mathcal{M}}_{s,a,u}$ may not be stationary, which means $\widetilde{r}_{0}^u,\widetilde{r}_{1}^u,\cdots,\widetilde{r}_{H-1}^u$ can be different.
\end{remark}

\begin{definition} (Feasible Set for $u$)
	For the auxiliary transition model $\widetilde{\mathcal{M}}_{s,a,u}$, $U^\pi_{s,a}$ is defined as the set of $u$ such that $\tilde{V}^\pi_{h,u}\in[0,H-h]^\mathcal{S},\forall h\in[H]$ and $U^*_{s,a}$ is defined as the set of $u$ such that $\tilde{V}^*_{h,u}\in[0,H-h]^\mathcal{S},\forall h\in[H]$.
\end{definition}

\begin{remark}
	$u$ that satisfies $0\leq r+u_h\Phi^{s,a}\leq1,\forall h\in[H]$ is in both $U^\pi_{s,a}$ and $U^*_{s,a}$ for arbitrary $\pi$. Immediately we have $\mathbf{0}\in U^\pi_{s,a}$ and $\mathbf{0}\in U^*_{s,a}$ for arbitrary $\pi$. Note that both $U^\pi_{s,a}$ and $U^*_{s,a}$ are independent of $\widehat{P}(s,a)$ and are intervals.
\end{remark}

\paragraph{Notations}
For simplicity, we ignore $(s,a)$ in functions of auxiliary transition model $\mathcal{M}_{s,a,u}$. We use $\widetilde{V}^\pi_{h,u},\widetilde{Q}^\pi_{h,u}$ to denote value function and Q-function in step $h$ and $\widetilde{\pi}^*_u$ to be the optimal policy in $\mathcal{M}_{s,a,u}$. 

\begin{lemma}
	For a $\epsilon_{\mathrm{PS}}$-optimal policy $\widehat{\pi}=\left(\widehat{\pi}_0,\cdots,\widehat{\pi}_{H-1}\right)$ in $\widehat{\mathcal{M}}$, we have
	$$\left\|Q_0^*-Q_0^{\widehat{\pi}}\right\|_\infty\leq\left\|V_0^*-\widehat{V}_0^{\pi^*}\right\|_\infty+\left\|\widehat{V}_0^{\widehat{\pi}}-V_0^{\widehat{\pi}}\right\|_\infty+\epsilon_\mathrm{PS}.$$
	\label{l15}
\end{lemma}

\begin{proof}
	\begin{align*}
	0<Q_0^*-Q_0^{\widehat{\pi}}&=Q_0^*-\widehat{Q}_0^{\pi^*}+\widehat{Q}_0^{\pi^*}-\widehat{Q}_0^*+\widehat{Q}_0^*-\widehat{Q}_0^{\widehat{\pi}}+\widehat{Q}_0^{\widehat{\pi}}-Q_0^{\widehat{\pi}}\\
	&\leq\left|Q_0^*-\widehat{Q}_0^{\pi^*}\right|+0+\epsilon_\mathrm{PS}+\left|\widehat{Q}_0^{\widehat{\pi}}-Q_0^{\widehat{\pi}}\right|,
	\end{align*}
	which implies $\left\|Q_0^*-Q_0^{\widehat{\pi}}\right\|_\infty\leq\left\|Q_0^*-\widehat{Q}_0^{\pi^*}\right\|_\infty+\left\|\widehat{Q}_0^{\widehat{\pi}}-Q_0^{\widehat{\pi}}\right\|_\infty+\epsilon_\mathrm{PS}.$
\end{proof}

\begin{lemma}
	For FHMDP, we have
	$$Q_0^\pi-\widehat{Q}_0^{\pi}=\sum_{h=0}^{H-1}\prod_{i=0}^{h-1}P^{\pi_i}(P-\widehat{P})\widehat{V}_{h+1}^{\pi}.$$
	\label{l16}
\end{lemma}

\begin{proof}
Using Bellman equation, we have
    \begin{align*}
        Q_0^\pi-\widehat{Q}_0^{\pi}&=(r+P V_1^\pi)-(r+\widehat{P} \widetilde{V}_1^\pi)\\
        &=P(V_1^\pi-\widehat{V}_1^\pi)+(P-\widehat{P})\widetilde{V}_1^\pi\\
        &=P^{\pi_0}(Q_1^\pi-\widehat{Q}_1^\pi)+(P-\widehat{P})\widetilde{V}_1^\pi\\
        &=P^{\pi_0}P^{\pi_1}(Q_2^\pi-\widehat{Q}_2^\pi)+P^{\pi_0}(P-\widehat{P})\widetilde{V}_2^\pi+(P-\widehat{P})\widetilde{V}_1^\pi\\
        &=\sum_{h=0}^{H-1}\prod_{i=0}^{h-1}P^{\pi_i}(P-\widehat{P})\widehat{V}_{h+1}^{\pi}.
    \end{align*}
The last equality is derived by iteratively expand $Q_h^\pi-\widehat{Q}_h^{\pi}$.
\end{proof}

\begin{lemma}
	For any policy $\pi$ and $V^\pi$ is the value function in a MDP with transition $P$,
	$$\left\|\sum_{h=0}^{H-1}\prod_{i=0}^{h-1}P^{\pi_i}\sqrt{Var_P(V_{h+1}^\pi)}\right\|_\infty\leq\sqrt{2H^3}.$$
	\label{l17}
\end{lemma}

\begin{proof}
    The proof is similar to the case in DMDP and can be found in \cite{sidford2018variance}. 
\end{proof}

\begin{lemma}
	Let $u_h^\pi=\left(\widehat{P}(s,a)-P(s,a)\right)\widehat{V}_{h+1}^\pi,\forall h\in[H]$ and $u_h^*=\left(\widehat{P}(s,a)-P(s,a)\right)\widehat{V}_{h+1}^*,\forall h\in[H]$, then we have
	$$\widehat{Q}_h^\pi=\widetilde{Q}^\pi_{h,u^\pi},\ \widehat{Q}_h^*=\widetilde{Q}^{\widehat{\pi}}_{h,u^*}=\widetilde{Q}^*_{h,u^*}, \left|u_h^\pi\right|\leq H-h-1,\left|u_h^*\right|\leq H-h-1,\forall h\in[H].$$
	\label{l18}
\end{lemma}

\begin{proof}
    We provethis argument by mathematical induction. When $h=H-1$, we have $u_h^\pi=\left(\widehat{P}(s,a)-P(s,a)\right)\widehat{V}_H^\pi=0$ and $\widehat{Q}_h^\pi=r=r+u^\pi_h\Phi^{s,a}=\widetilde{Q}_{h,u^\pi}^\pi$.
    
    If the argument holds for $h+1$, then we have
    \begin{align*}
        \widehat{Q}_h^\pi&=r+\widehat{P}\widehat{V}_{h+1}^\pi\\
        &=r+\widehat{P}\widetilde{V}_{h+1}^\pi\\
        &=r+\Phi\left(\widehat{P}_\mathcal{K}-\widetilde{P}_\mathcal{K}\right)\widehat{V}_{h+1}^\pi+\widetilde{P}\widetilde{V}_{h+1}^\pi\\
        &=r+\left(\widehat{P}(s,a)-P(s,a)\right)\widehat{V}_{h+1}^\pi\Phi^{s,a}+\widetilde{P}\widetilde{V}_{h+1}^\pi\\
        &=r+u^\pi_h\Phi^{s,a}+\widetilde{P}\widetilde{V}_{h+1}^\pi\\
        &=\widetilde{Q}^\pi_{h,u^\pi}.
    \end{align*}
    As $\widehat{V}_{h+1}^\pi\in[0,H-h-1]$, we have $u_h^\pi=(\widehat{P}(s,a)-P(s,a))\widehat{V}_{h+1}^\pi\in[h+1-H,H-h-1]$. The proof for $\widehat{\pi}^*$ is identical. 
\end{proof}

\begin{lemma}
	For all $u,u'\in\mathbb{R}^H$ and policy $\pi$,
	$$\left\|\widetilde{Q}^\pi_{h,u}-\widetilde{Q}^\pi_{h,u'}\right\|_\infty\leq(H-h)\left\|u-u'\right\|_\infty,\left\|\widetilde{Q}^*_{h,u}-\widetilde{Q}^*_{h,u'}\right\|_\infty\leq (H-h)\left\|u-u'\right\|_\infty.$$
	\label{l19}
\end{lemma}

\begin{proof}
    We prove the first claim:
    \begin{align*}
        \left\|\widetilde{Q}^\pi_{h,u}-\widetilde{Q}^\pi_{h,u'}\right\|_\infty&=\left\|\left(r+u_h\Phi^{s,a}+\widetilde{P}V^\pi_{h,u}\right)-\left(r+u'_h\Phi^{s,a}+\widetilde{P}V^\pi_{h,u'}\right)\right\|_\infty\\
        &\leq\left\|(u_h-u'_h)\Phi^{s,a}\right\|_\infty+\left\|\widetilde{P}^\pi\left(\widetilde{Q}^\pi_{h+1,u}-\widetilde{Q}^\pi_{h+1,u'}\right)\right\|_\infty\\
        &\leq\left|u_h-u'_h\right|+\left\|\widetilde{Q}^\pi_{h+1,u}-\widetilde{Q}^\pi_{h+1,u'}\right\|_\infty\\
        &\leq\sum_{i=h}^{H-1}\left|u_i-u'_i\right|\\
        &\leq(H-h)\left\|u-u'\right\|_\infty\\
    \end{align*}
    The proof for the second claim is identical.
\end{proof}

\begin{lemma}
	For a given finite set $B_{s,a}^{\pi^*}\subset U_{s,a}^{\pi^*}\cap[-H,H]^{H}$ and $\delta\geq0$, with probability greater than $1-\delta$, it holds for all $u\in B_{s,a}^\pi$ that
	$$\left|\left(P(s,a)-\widehat{P}(s,a)\right)\widetilde{V}^{\pi^*}_{h,u}\right|\leq\sqrt{\frac{2\log\left(4\left|B_{s,a}^{\pi^*}\right|/\delta\right)}{N}}\sqrt{Var_{s,a}(\widetilde{V}_{h,u}^{\pi^*})}+\frac{2\log\left(4\left|B_{s,a}^{\pi^*}\right|/\delta\right)(H-h)}{3N}.$$
	Similarly, For a given finite set $B_{s,a}^*\subset U_{s,a}^*\cap[-H,H]^H$ and $\delta\geq0$, with probability greater than $1-\delta$, it holds for all $u\in B_{s,a}^*$ that
	$$\left|\left(P(s,a)-\widehat{P}(s,a)\right)\widetilde{V}^*_{h,u}\right|\leq\sqrt{\frac{2\log\left(4\left|B_{s,a}^*\right|/\delta\right)}{N}}\sqrt{Var_{s,a}(\widetilde{V}_{h,u}^*)}+\frac{2\log\left(4\left|B_{s,a}^*\right|/\delta\right)(H-h)}{3N}.$$
	\label{l20}
\end{lemma}

\begin{proof}
	This is the direct application of Beinstein's inequality as $\widetilde{V}^\pi_{h,u}$ and $\widetilde{V}^*_{h,u}$ is independent of $\widehat{P}(s,a)$.
\end{proof}

\begin{lemma}
	For a given finite set $B_{s,a}^{\pi^*}\subset U_{s,a}^{\pi^*}\cap[-H,H]^{H}$ and $B_{s,a}^*\subset U_{s,a}^*\cap[-H,H]^H$ and $\delta\geq0$, with probability greater than $1-2H\delta$, it holds for all $u\in B_{s,a}^{\pi^*}$ and $h\in[H]$ that
	\begin{multline*}
	\left|\left(P(s,a)-\widehat{P}(s,a)\right)\widehat{V}_h^{\pi^*}\right|\leq\sqrt{\frac{2\log\left(4\left|B_{s,a}^{\pi^*}\right|/\delta\right)}{N}}\sqrt{Var_{s,a}(\widehat{V}_h^{\pi^*})}+\frac{2\log\left(4|B_{s,a}^{\pi^*}|/\delta\right)(H-h)}{3N}\\+\min_{u\in B_{s,a}^{\pi^*}}\left\|u^{\pi^*}-u\right\|_\infty(H-h)\left(1+\sqrt{\frac{2\log\left(4\left|B_{s,a}^{\pi^*}\right|/\delta\right)}{N}}\right),
	\end{multline*}
	\begin{multline*}
	\left|\left(P(s,a)-\widehat{P}(s,a)\right)\widehat{V}_h^*\right|\leq\sqrt{\frac{2\log\left(4\left|B_{s,a}^{*}\right|/\delta\right)}{N}}\sqrt{Var_{s,a}(\widehat{V}_h^*)}+\frac{2\log\left(4\left|B_{s,a}^{*}\right|/\delta\right)(H-h)}{3N}\\
	+\min_{u\in B_{s,a}^*}\left\|u^*-u\right\|_\infty(H-h)\left(1+\sqrt{\frac{2\log\left(4\left|B_{s,a}^*\right|/\delta\right)}{N}}\right).
	\end{multline*}
	\label{l21}
\end{lemma}

\begin{proof}
	We have
	\begin{align*}
	&\left|\left(P(s,a)-\widehat{P}(s,a)\right)\widehat{V}_h^{\pi^*}\right|
	\\\overset{\mathrm{(a)}}{\leq}&\left|\left(P(s,a)-\widehat{P}(s,a)\right)\widetilde{V}^{\pi^*}_{h,u}\right|+\left|\left(P(s,a)-\widehat{P}(s,a)\right)\left(\widehat{V}^{\pi^*}_h-\widetilde{V}^{\pi^*}_{h,u}\right)\right|\\
	\overset{\mathrm{(b)}}{\leq} &\sqrt{\frac{2\log\left(4\left|B_{s,a}^{\pi^*}\right|/\delta\right)}{N}}\sqrt{Var_{s,a}(\widetilde{V}_{h,u}^{\pi^*})}+\frac{2\log\left(4\left|B_{s,a}^{\pi^*}\right|/\delta\right)(H-h)}{3N}+\left\|\widehat{V}_h^{\pi^*}-\widetilde{V}^{\pi^*}_{h,u}\right\|_\infty\\
	\overset{\mathrm{(c)}}{\leq} &\sqrt{\frac{2\log\left(4\left|B_{s,a}^{\pi^*}\right|/\delta\right)}{N}}\left(\sqrt{Var_{s,a}(\widehat{V}_h^{\pi^*})}+\sqrt{Var_{s,a}\left(\widetilde{V}_{h,u}^{\pi^*}-\widehat{V}_h^{\pi^*}\right)}\right)\\&\ \ \ +\frac{2\log\left(4\left|B_{s,a}^{\pi^*}\right|/\delta\right)(H-h)}{3N}+\left\|\widehat{V}_h^{\pi^*}-\widetilde{V}^{\pi^*}_{h,u}\right\|_\infty\\
	\overset{\mathrm{(d)}}{\leq} &\sqrt{\frac{2\log\left(4\left|B_{s,a}^{\pi^*}\right|/\delta\right)}{N}}\sqrt{Var_{s,a}(\widehat{V}_h^{\pi^*})} +\frac{2\log\left(4\left|B_{s,a}^{\pi^*}\right|/\delta\right)(H-h)}{3N}\\&\ \ \ +\left\|\widehat{V}_h^{\pi^*}-\widetilde{V}^{\pi^*}_{h,u}\right\|_\infty\left(1+\sqrt{\frac{2\log\left(4\left|B_{s,a}^\pi\right|/\delta\right)}{N}}\right)\\
	\overset{\mathrm{(e)}}{\leq} &\sqrt{\frac{2\log\left(4\left|B_{s,a}^{\pi^*}\right|/\delta\right)}{N}}\sqrt{Var_{s,a}(\widehat{V}_h^{\pi^*})} +\frac{2\log\left(4\left|B_{s,a}^{\pi^*}\right|/\delta\right)(H-h)}{3N}\\&\ \ \ +\left\|u^{\pi^*}-u\right\|_\infty(H-h)\left(1+\sqrt{\frac{2\log\left(4\left|B_{s,a}^{\pi^*}\right|/\delta\right)}{N}}\right).
	\end{align*}
	(a) is due to triangle inequality, (b) is from Lemma \ref{l20}, (c) is from Lemma \ref{l5}, (d) is due to the fact that $\sqrt{Var(V)}\leq V$ and (e) is from Lemma \ref{l19}. As this equality holds for all $u\in B_{s,a}^{\pi^*}$, we can take minimum in the RHS, which proves the first claim. The second claims can proved in the same manner.
\end{proof}

\begin{lemma}
	For any given $\epsilon$ and all $(s,a)\in \mathcal{K}$, with probability larger than $1-2H\delta$,
	\begin{multline*}
	\left|\left(P(s,a)-\widehat{P}(s,a)\right)\widehat{V}_h^{\pi^*}\right|\leq\sqrt{\frac{2H\log(32KH^3/\delta\epsilon)}{N}}\sqrt{Var_{s,a}(\widehat{V}_h^{\pi^*})}+\frac{2H^2\log(32KH^3/\delta\epsilon)}{3N}
	\\+\left(\sqrt{\frac{2H\log(32KH^3/\delta\epsilon)}{N}}+1\right)\frac{\epsilon}{4H},
	\end{multline*}
	\begin{multline*}
	\left|\left(P(s,a)-\widehat{P}(s,a)\right)\widehat{V}_h^{*}\right|\leq\sqrt{\frac{2H\log(32KH^3/\delta\epsilon)}{N}}\sqrt{Var_{s,a}(\widehat{V}_h^*)}+\frac{2H^2\log(32KH^3/\delta\epsilon)}{3N}
	\\+\left(\sqrt{\frac{2H\log(32KH^3/\delta\epsilon)}{N}}+1\right)\frac{\epsilon}{4H}.
	\end{multline*}
	For simplicity, we set $c=c(\delta,H,\epsilon)=2\log(64KH^4/\delta\epsilon)$.
	\label{l22}
	\end{lemma}

\begin{proof}
	We set $B_{s,a}^{\pi^*}$ to be the evenly spaced elements in the interval $U^{\pi^*}_{s,a}\cap[-H,H]^{H}$ and $\left|B_{s,a}^{\pi^*}\right|=\left(\frac{4H^3}{\epsilon}\right)^H$. Then for any $u'\in U^{\pi^*}_{s,a}\cap[-H,H]^{H}$, we have $\min_{u\in B_{s,a}^{\pi^*}}\left\|u'-u\right\|_\infty\leq\epsilon/4H^2$. Note that $u^{\pi^*}\in U^*_{s,a}\cap[-H,H]^{H}$. Combining this with Lemma \ref{l21} implies the result. Similarly we can prove it for $B_{s,a}^*$.
\end{proof}

\begin{lemma}
	For any given $\epsilon$ and all $(s,a)\in \mathcal{K}$, with probability larger than $1-2H\delta$,
	\begin{multline*}
	\left|\left(P(s,a)-\widehat{P}(s,a)\right)\widehat{V}_h^{\pi^*}\right|\leq\sqrt{\frac{2\min\{K,|\mathcal{S}|\}\log(32KH^3/\delta\epsilon)}{N}}\sqrt{Var_{s,a}(\widehat{V}_h^{\pi^*})}\\+\frac{2\min\{K,|\mathcal{S}|\}H\log(32KH^3/\delta\epsilon)}{3N}
	+\left(\sqrt{\frac{2\min\{K,|\mathcal{S}|\}\log(32KH^3/\delta\epsilon)}{N}}+1\right)\frac{\epsilon}{4H},
	\end{multline*}
	\begin{multline*}
	\left|\left(P(s,a)-\widehat{P}(s,a)\right)\widehat{V}_h^{*}\right|\leq\sqrt{\frac{2\min\{K,|\mathcal{S}|\}\log(32KH^3/\delta\epsilon)}{N}}\sqrt{Var_{s,a}(\widehat{V}_h^*)}\\+\frac{2\min\{K,|\mathcal{S}|\}H\log(32KH^3/\delta\epsilon)}{3N}
	+\left(\sqrt{\frac{2\min\{K,|\mathcal{S}|\}\log(32KH^3/\delta\epsilon)}{N}}+1\right)\frac{\epsilon}{4H}.
	\end{multline*}
	\label{l23}
\end{lemma}

\begin{proof}
	Lemma \ref{l22} is proved by constructing a $\epsilon/4H$-net on $\widehat{V}_h^{\pi^*}$ via auxiliary MDP. Note that
	$$\widehat{V}^*_h=\max_a\widehat{Q}^*_h=\max_a(r+\Phi \widehat{P}_\mathcal{K}\widehat{V}^*_{h+1}),$$
	which means $\widehat{V}^*_h$ lies in a $K$-dimensional manifold in $[0,H]^
	\mathcal{S}$. We can make an $\epsilon/4H$-net on this manifold with $O(\frac{4H^3}{\epsilon})^{\min\{K,|\mathcal{S}|\}}$ points. With similar analysis as Lemma \ref{l22}, we can prove this claim.
\end{proof}

	


\begin{proof}[Proof of Theorem 3]
    From Lemma \ref{l15}, with probability larger than $1-\delta$, we have
    \begin{align*}
        \left\|Q_0^*-Q_0^{\widehat{\pi}}\right\|_\infty&\leq\left\|Q_0^*-\widehat{Q}_0^{\pi^*}\right\|_\infty+\left\|\widehat{Q}_0^{\widehat{\pi}}-Q_0^{\widehat{\pi}}\right\|_\infty+\epsilon_\mathrm{PS}\\
        &\leq\frac{1}{1-\alpha}\bigg[2\bigg(\sqrt{\frac{cH^3\min\{H,K,|\mathcal{S}|\}}{N}}+\frac{cH^2\min\{H,K,|\mathcal{S}|\}}{3N}\\
	&+\left(\sqrt{\frac{c}{N}}+1\right)\frac{\epsilon}{4}\bigg)+\left(\sqrt{\frac{c}{N}}+1\right)\epsilon_\mathrm{PS}H\bigg]+\epsilon_\mathrm{PS},\\
    \end{align*}
    where $\alpha=\alpha(\delta,H,\epsilon,N)=\sqrt{\frac{c(\delta,H,\epsilon)H^2}{N}}$ and the second inequality can be derived as in DMDP from Lemma \ref{l22} and Lemma \ref{l23}.
    For $N\geq\frac{C\log(CKH\delta^{-1}\epsilon^{-1})H^3\min\{H,K,|\mathcal{S}|\}}{\epsilon^2}$ with proper constant $C$, we have $\frac{1}{1-\alpha}(\sqrt{\frac{c}{N}}+1)\leq2$, thus we have
    \[\left\|V_0^*-V_0^{\widehat{\pi}}\right\|_\infty\leq\epsilon+3\epsilon_\mathrm{PS}H.\]
\end{proof}

\subsection{Sample Complexity for 2-TBSG}

The value concentration becomes a little more tricky in 2-TBSG. The proof is similar to the case for DMDP, which only differs in the attendance of counter policy. The counter policy in 2-TBSG can be seen as the optimal policy in a DMDP induced by the policy of the opponent.

\begin{definition}
	(Auxiliary Model) For a estimated transition model $\widehat{\mathcal{M}}=(\mathcal{S}_1,\mathcal{S}_2,\mathcal{A},\widehat{P}=\Phi\widehat{P}_\mathcal{K},r,\gamma)$ and a given anchor state pair $(s,a)$, the auxiliary transition model is $\widetilde{\mathcal{M}}_{s,a,u}=(\mathcal{S}_1,\mathcal{S}_2,\mathcal{A},\widetilde{P}=\Phi\widetilde{P}_K,r+u\Phi^{s,a},\gamma)$, where
	$$\widetilde{P}_\mathcal{K}(s',a')=
	\begin{cases}
	\widehat{P}(s',a')& \mathrm{if}\ (s',a')\neq(s,a)\\
	P(s,a)& \mathrm{otherwise},
	\end{cases}$$
	$\Phi^{s,a}$ is the column of $\Phi$ that corresponds to anchor state $(s,a)$ and $u$ is a variable that will be determined latter.
\end{definition}

\paragraph{Notations}
For simplicity, we ignore $(s,a)$ in functions of auxiliary transition model $\mathcal{M}_{s,a,u}$. $c(\pi_1),\widehat{c}(\pi_1),\widetilde{c}_u(\pi_1)$ are the counter policies for $\pi_1$ in $\mathcal{M},\widehat{\mathcal{M}},\widetilde{\mathcal{M}}_{s,a,u}$. When it is clear in the context, we use $c$ as the counter policy function for $\pi_2$ as well. $\pi^*=(\pi_1^*,\pi_2^*),\widehat{\pi}^*=(\widehat{\pi}_1^*,\widehat{\pi}_2^*),\widetilde{\pi}^*=(\widetilde{\pi}_1^*,\widetilde{\pi}_2^*)$ are the equilibrium policies in $\mathcal{M},\widehat{\mathcal{M}},\widetilde{\mathcal{M}}_{s,a,u}$. We use $\widetilde{V}^\pi_u,\widetilde{Q}^\pi_u$ to denote value function and Q-function and $\widetilde{\pi}_u$ to be the optimal policy in $\mathcal{M}_{s,a,u}$.

\begin{definition} (Feasible Set for $u$)
	For the auxiliary transition model $\widetilde{M}$, $U^\pi_{s,a}$ is defined as the set of $u$ so that $\widetilde{V}^\pi_u\in[0,1/(1-\gamma)]^\mathcal{S}$ and $U^*_{s,a}$ is defined as the set of $u$ so that $\widetilde{V}^*_u\in[0,1/(1-\gamma)]^\mathcal{S}$.
\end{definition}

\begin{lemma}
	$$V^{c(\pi_2),\pi_2}\geq V^{\pi_1^*,\pi_2^*},\ V^{\pi_1,c(\pi_1)}\leq V^{\pi_1^*,\pi_2^*}.$$
	\label{l24}
\end{lemma}

\begin{lemma}
	$$Q^{\pi_1,c(\pi_1)}(s,a)\geq Q^{\pi_1,c(\pi_1)}(s,c(\pi_1)(s)),\ \forall s\in\mathcal{S}_2$$
	$$Q^{\pi_1^*,\pi_2^*}(s,a)\begin{cases}
	\geq Q^{\pi_1^*,\pi_2^*}(s,\pi_2^*(s))&\forall s\in\mathcal{S}_2\\
	\leq Q^{\pi_1^*,\pi_2^*}(s,\pi_1^*(s))&\forall s\in\mathcal{S}_1.
	\end{cases}$$
	These two equalities are also the sufficient condition for counter policy and equilibrium policy.
	\label{l25}
\end{lemma}

\begin{proof}
    The proof of Lemma \ref{l24} and \ref{l25} can be found in \cite{hansen2013strategy}.
\end{proof}

\begin{lemma}
	Let $\widehat{\pi}=(\widehat{\pi}_1,\widehat{\pi}_2)$ be a $\epsilon_\mathrm{PS}$-optimal policy in $\widehat{M}$.
	\begin{multline*}
		-(\left\|Q^{\widehat{c}(\pi_2^*),\pi_2^*}-\widehat{Q}^{\widehat{c}(\pi_2^*),\pi_2^*}\right\|_\infty+\left\|Q^{\widehat{\pi}_1,\widehat{\pi}_2}-\widehat{Q}^{\widehat{\pi}_1,\widehat{\pi}_2}\right\|_\infty+\epsilon_\mathrm{PS})\mathbf{1}\leq Q^*-Q^{\widehat{\pi}_1,\widehat{\pi}_2}\\\leq (\left\|Q^{\pi_1^*,\widehat{c}(\pi_1^*)}-\widehat{Q}^{\pi_1^*,\widehat{c}(\pi_1^*)}\right\|_\infty+\left\|Q^{\widehat{\pi}_1,\widehat{\pi}_2}-\widehat{Q}^{\widehat{\pi}_1,\widehat{\pi}_2}\right\|_\infty+\epsilon_\mathrm{PS})\mathbf{1}.
	\end{multline*}
	\label{l26}
\end{lemma}

\begin{proof}
    We prove the second inequality and the first one can be proved by symmetry.
    \begin{align*}
        &Q^*-Q^{\widehat{\pi}_1,\widehat{\pi}_2}\\
        =&Q^*-Q^{\pi_1^*,\widehat{c}(\pi_1^*)}+Q^{\pi_1^*,\widehat{c}(\pi_1^*)}-\widehat{Q}^{\pi_1^*,\widehat{c}(\pi_1^*)}+\widehat{Q}^{\pi_1^*,\widehat{c}(\pi_1^*)}-\widehat{Q}^*+\widehat{Q}^*-\widehat{Q}^{\widehat{\pi}}+Q^{\widehat{\pi}}-\widehat{Q}^{\widehat{\pi}}\\
        \leq &Q^{\pi_1^*,\widehat{c}(\pi_1^*)}-\widehat{Q}^{\pi_1^*,\widehat{c}(\pi_1^*)}+\widehat{Q}^*-\widehat{Q}^{\widehat{\pi}}+Q^{\widehat{\pi}}-\widehat{Q}^{\widehat{\pi}}\\
        \leq &\left(\left\|Q^{\pi_1^*,\widehat{c}(\pi_1^*)}-\widehat{Q}^{\pi_1^*,\widehat{c}(\pi_1^*)}\right\|_\infty+\left\|Q^{\widehat{\pi}_1,\widehat{\pi}_2}-\widehat{Q}^{\widehat{\pi}_1,\widehat{\pi}_2}\right\|_\infty+\epsilon_\mathrm{PS}\right)\mathbf{1}
    \end{align*}
\end{proof}


\begin{lemma}
	Let $u^{\pi_1}=\gamma\left(\widehat{P}(s,a)-P(s,a)\right)\widehat{V}^{\pi_1,\widehat{c}(\pi_1)},\ u^{\pi_2}=\gamma\left(\widehat{P}(s,a)-P(s,a)\right)\widehat{V}^{\widehat{c}(\pi_2),\pi_2},\ u^{*}=\gamma\left(\widehat{P}(s,a)-P(s,a)\right)\widehat{V}^{\widehat{\pi}_1^*,\widehat{\pi}_2^*}$, we have
	$$\widehat{Q}^{\pi_1,\widehat{c}(\pi_1)}=\widetilde{Q}^{\pi_1,\widehat{c}(\pi_1)}_{u^{\pi_1}}=\widetilde{Q}^{\pi_1,\widetilde{c}(\pi_1)}_{u^{\pi_1}},\widehat{Q}^{\widehat{c}(\pi_2),\pi_2}=\widetilde{Q}^{\widehat{c}(\pi_2),\pi_2}_{u^{\pi_2}}=\widetilde{Q}^{\widetilde{c}(\pi_2),\pi_2}_{u^{\pi_2}},\widehat{Q}^*=\widetilde{Q}^{\widehat{\pi}_1^*,\widehat{\pi}_2^*}_{u^*}=\widetilde{Q}^{*}_{u^*}.$$
	\label{l27}
\end{lemma}

\begin{proof}
The proof of the first equality in two arguments are identical to Lemma \ref{l6}. Combining with Lemma \ref{l25}, we have the second equality.
\end{proof}

\begin{lemma}
	$$\left\|\tilde{Q}^{\pi_1,\tilde{c}_{u_1}(\pi_1)}_{u_1}-\tilde{Q}^{\pi_1,\tilde{c}_{u_2}(\pi_1)}_{u_2}\right\|\leq |u_1-u_2|\frac{1}{1-\gamma},$$
	$$\left\|\tilde{Q}^{*}_{u_1}-\tilde{Q}^{*}_{u_2}\right\|\leq |u_1-u_2|\frac{1}{1-\gamma}$$
	\label{l28}
\end{lemma}

\begin{proof}
The proof is identical to Lemma \ref{l7}.
\end{proof}

\begin{lemma}
	With probability larger than $1-\delta$, and $\widehat{\pi}$ is a $\epsilon_\mathrm{PS}$-optimal policy in $\widehat{M}$.
	\begin{multline*}
	\left\|Q^{\pi_1^*,\widehat{c}(\pi_1^*)}-\widehat{Q}^{\pi_1^*,\widehat{c}(\pi_1^*)}\right\|_\infty\leq\frac{\gamma}{1-\alpha}\left(\sqrt{\frac{c}{N(1-\gamma)^3}}+\frac{c}{(1-\gamma)^23N}
	+\left(\sqrt{\frac{c}{N}}+1\right)\frac{\epsilon}{4}\right),
	\end{multline*}
	\begin{multline*}
	\left\|Q^{\widehat{\pi}}-\widehat{Q}^{\widehat{\pi}}\right\|_\infty\leq\frac{\gamma}{1-\alpha}\left(\sqrt{\frac{c}{N(1-\gamma)^3}}+\frac{c}{(1-\gamma)^23N}
	+\left(\sqrt{\frac{c}{N}}+1\right)\left(\frac{\epsilon}{4}+\frac{\epsilon_\mathrm{PS}}{1-\gamma}\right)\right).
	\end{multline*}
	where $c$ and $\alpha$ is defined as in Lemma 10 and Lemma 12.
	\label{l29}
\end{lemma}

\begin{proof}
The proof is identical to Lemma \ref{l12} as we have Lemma \ref{l27} and Lemma \ref{l28} in 2-TBSG, which is the counterpart of Lemma \ref{l6} and Lemma \ref{l7}.
\end{proof}

\begin{proof}[Proof of Theorem 4]
    From Lemma \ref{l26}, with probability larger than $1-\delta$, we have
    \begin{align*}
        &Q^*-Q^{\widehat{\pi}_1,\widehat{\pi}_2}
        \\\leq&\left(\left\|Q^{\pi_1^*,\widehat{c}(\pi_1^*)}-\widehat{Q}^{\pi_1^*,\widehat{c}(\pi_1^*)}\right\|_\infty+\left\|\widehat{Q}^{\widehat{\pi}}-Q^{\widehat{\pi}}\right\|_\infty+\epsilon_\mathrm{PS}\right)\mathbf{1}\\
        \leq&\bigg(\frac{\gamma}{1-\alpha}\left[2\left(\sqrt{\frac{c}{N(1-\gamma)^3}}+\frac{c}{(1-\gamma)^23N}
	+\left(\sqrt{\frac{c}{N}}+1\right)\frac{\epsilon}{4}\right)+\left(\sqrt{\frac{c}{N}}+1\right)\frac{\epsilon_\mathrm{PS}}{1-\gamma}\right]\\&+\epsilon_\mathrm{PS}\bigg)\mathbf{1}\\
    \end{align*}
    For $N\geq\frac{C\log(CK(1-\gamma)^{-1}\delta^{-1}\epsilon^{-1})}{(1-\gamma)^3\epsilon^2}$ with proper constant $C$, we have $\frac{\gamma}{1-\alpha}(\sqrt{\frac{c}{N}}+1)\leq2$, thus
    \[V^*-V^{\widehat{\pi}_1,\widehat{\pi}_2}\leq(\epsilon+\frac{3\epsilon_\mathrm{PS}}{1-\gamma})\mathbf{1}
    .\]
    By symmetry, we have 
    \[V^*-V^{\widehat{\pi}_1,\widehat{\pi}_2}\geq-(\epsilon+\frac{3\epsilon_\mathrm{PS}}{1-\gamma})\mathbf{1}.\]
    Thus we have
    \[\left\|V^*-V^{\widehat{\pi}_1,\widehat{\pi}_2}\right\|_\infty\leq\epsilon+\frac{3\epsilon_\mathrm{PS}}{1-\gamma},\]
    which completes the proof.

\end{proof}

\section{Sample Complexity in General Linear Case}
 In the general linear case, the empirical MDP can be a pseudo MDP and hence we cannot get the optimal policy in the empirical MDP. Even so, the value iteration solver can still be applied to the pseudo MDP and we prove that it is a sample efficient algorithm. The proof relies on an observation that a discounted MDP can be approximated by a discounted finite horizon MDP with horizon $H=O(1/(1-\gamma)\log\epsilon^{-1})$. We use $\widehat{V}_{h}^*$ to represent the result of operating value iteration for $H-h$ steps in the empirical model and $V_{h}^*$ to be the result in true model. The initial value is $\widehat{V}_H^*=V_H^*=0$. We use a similar analysis in FHMDP, but the pseudo MDP leads to some defect of previous proof and we revised the bound.
 
 \begin{lemma}
 The error of performing value iteration for $H$ steps can be bounded:
     $$\left\|\widehat{Q}^*_0-Q^*_0\right\|_\infty\leq\sum_{h=0}^{H-1}\gamma^{h+1}L\left\|\left(\widehat{P}_\mathcal{K}-P_\mathcal{K}\right)\widehat{V}_{h+1}^*\right\|_\infty.$$
     \label{l30}
 \end{lemma}
 
 \begin{proof}
 The proof is from iteratively using the following inequality.
     \begin{align*}
         \left\|\widehat{Q}^*_h-Q^*_h\right\|_\infty&=\left\|\left(r+\gamma\widehat{P}\widehat{V}^*_{h+1}\right)-\left(r+\gamma PV^*_{h+1}\right)\right\|_\infty\\
         &\leq\left\|\gamma P\left(\widehat{V}^*_{h+1}-V^*_{h+1}\right)\right\|_\infty+\left\|\gamma\left(\widehat{P}-P\right)\widehat{V}_{h+1}^*\right\|_\infty\\
         &\leq\gamma\left\|\widehat{Q}^*_{h+1}-Q^*_{h+1}\right\|_\infty+\gamma\left\|\left(\widehat{P}-P\right)\widehat{V}_{h+1}^*\right\|_\infty\\
         &=\gamma\left\|\widehat{Q}^*_{h+1}-Q^*_{h+1}\right\|_\infty+\gamma\left\|\Phi\left(\widehat{P}_\mathcal{K}-P_\mathcal{K}\right)\widehat{V}_{h+1}^*\right\|_\infty\\
         &\leq\gamma
         \left\|\widehat{Q}^*_{h+1}-Q^*_{h+1}\right\|_\infty+\gamma L\left\|\left(\widehat{P}_\mathcal{K}-P_\mathcal{K}\right)\widehat{V}_{h+1}^*\right\|_\infty.
     \end{align*}
 \end{proof}
 
 The auxiliary MDP technique can be used in the same way to analyze $(\widehat{P}-P)\widehat{V}_{h+1}^*$ as in Lemma \ref{l22}. Here we give two lemmas, which are counterparts of Lemma \ref{l19} and Lemma \ref{l20} for FHMDP. Note that the total variance technique is not applicable in pseudo MDP as Lemma \ref{l3} no longer holds. Therefore, we use Hoeffding's inequality to analyze the concentration.
 
 \begin{lemma}
	For all $u,u'\in\mathbb{R}^H$ and policy $\pi$,
	$$\left\|\widetilde{Q}^*_{h,u}-\widetilde{Q}^*_{h,u'}\right\|_\infty\leq (H-h)L^{H-h}\|u-u'\|_\infty.$$
	\label{l31}
\end{lemma}

\begin{proof}Similar to the proof of Lemma 19, we have
    \begin{align*}
        \left\|\widetilde{Q}^\pi_{h,u}-\widetilde{Q}^\pi_{h,u'}\right\|_\infty&=\left\|\left(r+u_h\Phi^{s,a}+\widetilde{P}V^\pi_{h,u}\right)-\left(r+u'_h\Phi^{s,a}+\widetilde{P}V^\pi_{h,u'}\right)\right\|_\infty\\
        &\leq\left\|(u_h-u'_h)\Phi^{s,a}\right\|_\infty+\left\|\widetilde{P}^\pi\left(\widetilde{Q}^\pi_{h+1,u}-\widetilde{Q}^\pi_{h+1,u'}\right)\right\|_\infty\\
        &=\left\|(u_h-u'_h)\Phi^{s,a}\right\|_\infty+\left\|\Phi\widetilde{P}_\mathcal{K}^\pi\left(\widetilde{Q}^\pi_{h+1,u}-\widetilde{Q}^\pi_{h+1,u'}\right)\right\|_\infty\\
        &\leq|u_h-u'_h|+L\left\|\widetilde{Q}^\pi_{h+1,u}-\widetilde{Q}^\pi_{h+1,u'}\right\|_\infty\\
        &\leq\sum_{i=h}^{H-1}L^{i-h}\left|u_i-u'_i\right|\\
        &\leq(H-h)L^{H-h}\left\|u-u'\right\|_\infty.\\
    \end{align*}
\end{proof}

\begin{lemma}
For a given finite set $B_{s,a}^*\subset U_{s,a}^*\cap[-H,H]^H$ and $\delta\geq0$, with probability greater than $1-\delta$, it holds for all $u\in B_{s,a}^*$ that
	$$\left|\left(P(s,a)-\widehat{P}(s,a)\right)\widetilde{V}^*_{h,u}\right|\leq\sqrt{\frac{H^2\log\left(2\left|B_{s,a}^*\right|/\delta\right)}{2N}}.$$
	\label{l32}
\end{lemma}

\begin{proof}
	This is the direct application of Hoeffding's inequality as $\widetilde{V}^\pi_{h,u}$ and $\widetilde{V}^*_{h,u}$ is independent of $\widehat{P}(s,a)$.
\end{proof}

\begin{lemma}
	For a given finite set $B_{s,a}^{\pi^*}\subset U_{s,a}^{\pi^*}\cap[-H,H]^{H}$ and $B_{s,a}^*\subset U_{s,a}^*\cap[-H,H]^H$ and $\delta\geq0$, with probability greater than $1-H\delta$, it holds for all $u\in B_{s,a}^{*},h\in[H]$ that
	$$
	\left|\left(P(s,a)-\widehat{P}(s,a)\right)\widehat{V}_h^*\right|\leq\sqrt{\frac{H^2\log\left(2\left|B_{s,a}^*\right|/\delta\right)}{2N}}+
	\min_{u\in B_{s,a}^*}\|u^*-u\|_\infty(H-h)L^{H-h}.
	$$
	\label{l33}
\end{lemma}

\begin{proof}
	Combining Lemma \ref{l31} and Lemma \ref{l32}, we have
	\begin{align*}
	\left|\left(P(s,a)-\widehat{P}(s,a)\right)\widehat{V}_h^{*}\right|&=\left|\left(P(s,a)-\widehat{P}(s,a)\right)\left(\widetilde{V}^*_{h,u}+\widehat{V}_h^{*}-\widetilde{V}^*_{h,u}\right)\right|\\
	&\leq\left|\left(P(s,a)-\widehat{P}(s,a)\right)\widetilde{V}^*_{h,u}\right|+\left\|\widehat{V}_h^{*}-\widetilde{V}^*_{h,u}\right\|_\infty\\
	&\leq\sqrt{\frac{H^2\log\left(2\left|B_{s,a}^*\right|/\delta\right)}{2N}}+\|u^*-u\|_\infty(H-h)L^{H-h}.
	\end{align*}
As this equality holds for all $u\in B_{s,a}^{\pi^*}$, we can take minimum in the RHS, which proves the first claim.
\end{proof}

\begin{lemma}
	For any given $\epsilon$ and all $(s,a)\in \mathcal{K}$, with probability larger than $1-H\delta$,
	$$
	\left|\left(P(s,a)-\widehat{P}(s,a)\right)\widehat{V}_h^{*}\right|\leq\sqrt{\frac{H^4\log(8H^3L^2/\delta\epsilon)}{2N}}+\frac{\epsilon}{4HL}.
	$$
	\label{l34}
\end{lemma}

\begin{proof}
	We set $B_{s,a}^{\pi^*}$ to be the evenly spaced elements in the interval $U^{\pi^*}_{s,a}\cap[-H,H]^{H}$ and $\left|B_{s,a}^{\pi^*}\right|=(\frac{4H^3L^{2H}}{\epsilon})^H$. Then for any $u'\in U^{\pi^*}_{s,a}\cap[-H,H]^{H}$, we have $\min_{u\in B_{s,a}^{\pi^*}}\left\|u'-u\right\|_\infty\leq\epsilon/4H^2L^{2H}$. Note that $u^{\pi^*}\in U^*_{s,a}\cap[-H,H]^{H}$. Combining this with Lemma 33 implies the result.
\end{proof}

\begin{proof}[Proof of Theorem 5]
From Lemma \ref{l34}, with probability larger than $1-\delta$, we have
    \begin{align*}
        \left\|\widehat{V}^*_0-V^*_0\right\|_\infty&\leq\left\|\widehat{Q}^*_0-Q^*_0\right\|_\infty\\
        &\leq\sum_{h=0}^{H-1}\gamma^{h+1}L\left\|\left(\widehat{P}_\mathcal{K}-P_\mathcal{K}\right)\widehat{V}_{h+1}^*\right\|_\infty\\
        &\leq\sum_{h=0}^{H-1}\gamma^{h+1}L\left(\sqrt{\frac{H^4\log(8H^4L^2/\delta\epsilon)}{2N}}+\frac{\epsilon}{4HL}\right)\\
        &\leq\sqrt{\frac{H^6L^2\log(8H^3L^2/\delta\epsilon)}{2N}}+\frac{\epsilon}{4}
    \end{align*}
    Choosing $N\geq CH^6L^2\log(CKHL/\delta\epsilon)(\epsilon(1-\gamma))^{-2}$, we have $\left\|\widehat{V}^*_0-V^*_0\right\|_\infty\leq\epsilon(1-\gamma)/2$. Now we replace $H$ with $O((1-\gamma)^{-1}\log(1/\epsilon))$ and we have $\left\|V^*_0-V^*\right\|_\infty\leq\epsilon(1-\gamma)/2$ by the convergence of value iteration. So we have $\left\|\widehat{V}^*_0-V^*\right\|_\infty\leq\epsilon(1-\gamma)$, which implies the greedy policy with respect to $\widehat{V}^*_0$ is an $\epsilon$-greedy policy in the true MDP \cite{singh1994upper}.

\end{proof}

\end{document}